\documentclass[pdflatex,sn-apa,iicol]{cls/sn-jnl}

\usepackage{graphicx,tabularx,adjustbox}%
\usepackage{multirow}%
\usepackage{amsmath,amssymb,amsfonts,bbm}%
\usepackage{amsfonts,amstext,dsfont,mathtools}%
\usepackage{amsthm}%
\usepackage{mathrsfs}%
\usepackage[title]{appendix}%
\usepackage{xcolor}%
\usepackage{textcomp}%
\usepackage{manyfoot}%
\usepackage{booktabs}%
\let\orcidlogo\undefined
\usepackage{orcidlink}
\usepackage{algorithmicx}%
\usepackage{listings}%
\usepackage{xspace}
\usepackage{epsfig} %
\usepackage{tikz}
\usetikzlibrary{fadings}

\usepackage{hhline}
\usepackage{xcolor}
\usepackage{float}
\usepackage{subcaption}

\usepackage[ruled,vlined,noend,linesnumbered]{algorithm2e}
\usepackage{algcompatible}%
\newcommand{\bubblestar}{Bubble$^\star$\xspace}
\newcommand{\norm}[1]{\left\| #1 \right\|}

\DeclareMathOperator*{\argmin}{arg\,min}

\DeclareMathOperator*{\diag}{diag}

\newcommand{\bfa}{\mathbf{a}}

\newcommand{\bfc}{\mathbf{c}}

\newcommand{\bfe}{\mathbf{e}}
\newcommand{\bff}{\mathbf{f}}
\newcommand{\bfg}{\mathbf{g}}

\newcommand{\bfj}{\mathbf{j}}
\newcommand{\bfk}{\mathbf{k}}

\newcommand{\bfo}{\mathbf{o}}
\newcommand{\bfp}{\mathbf{p}}
\newcommand{\bfq}{\mathbf{q}}

\newcommand{\bfu}{\mathbf{u}}
\newcommand{\bfv}{\mathbf{v}}
\newcommand{\bfw}{\mathbf{w}}
\newcommand{\bfx}{\mathbf{x}}
\newcommand{\bfy}{\mathbf{y}}
\newcommand{\bfz}{\mathbf{z}}

\newcommand{\bfnu}{\boldsymbol{\nu}}

\newcommand{\bftau}{\boldsymbol{\tau}}

\newcommand{\bfomega}{\boldsymbol{\omega}}
\newcommand{\bfxi}{\boldsymbol{\xi}}

\newcommand{\bfG}{\mathbf{G}}

\newcommand{\bfT}{\mathbf{T}}

\newcommand{\bbB}{\mathbb{B}}

\newcommand{\bbR}{\mathbb{R}}

\newcommand{\bbZ}{\mathbb{Z}}

\newcommand{\calB}{\mathcal{B}}
\newcommand{\calC}{\mathcal{C}}

\newcommand{\calF}{\mathcal{F}}

\newcommand{\calJ}{\mathcal{J}}

\newcommand{\calL}{\mathcal{L}}

\newcommand{\calN}{\mathcal{N}}
\newcommand{\calO}{\mathcal{O}}
\newcommand{\calP}{\mathcal{P}}

\newcommand{\calR}{\mathcal{R}}
\newcommand{\calS}{\mathcal{S}}

\newcommand{\calV}{\mathcal{V}}

\newcommand{\calX}{\mathcal{X}}

\newtheorem{theorem}{Theorem}

\newtheorem{corollary}{Corollary}

\newtheorem{definition}{Definition}

\def\htwomapping{$\text{H}_2$-Mapping\xspace}
\def\oren{OREN\xspace}

\begin{document}

\def\papertitle{From Distances to Trajectories: Real-Time Signed Distance Function Mapping and Distance-Accelerated Motion Planning for UAVs}
\title{\Large\bfseries \papertitle}

\author*[1]{\fnm{Jason} \sur{Stanley}\orcid{0009-0002-5425-4834}}\email{jtstanle@ucsd.edu}
\equalcont{These authors contributed equally to this work.}
\author*[1]{\fnm{Zhirui} \sur{Dai}\orcid{0000-0001-8362-3796}}\email{zhdai@ucsd.edu}
\equalcont{These authors contributed equally to this work.}

\author[1]{\fnm{Qihao} \sur{Qian}\orcid{0009-0005-7286-9084}}\email{q2qian@ucsd.edu}

\author[1]{\fnm{Tzu-Chin} \sur{Ho}\orcid{0009-0002-8051-1072}}\email{tzh005@ucsd.edu}

\author[1]{\fnm{Tianxing} \sur{Fan}\orcid{0009-0003-0820-7735}}\email{t2fan@ucsd.edu}

\author[2]{\fnm{Siddharth} \sur{Saha}}
\author[2]{\fnm{Christopher} \sur{Barngrover}}
\author[1]{\fnm{Ki Myung Brian} \sur{Lee}\orcid{0000-0003-1449-2125}}\email{kmblee@ucsd.edu}
\author[1]{\fnm{Nikolay} \sur{Atanasov}\orcid{0000-0003-0272-7580}}\email{natanasov@ucsd.edu}

\affil[1]{\orgdiv{Electrical and Computer Engineering Department}, \orgname{UC San Diego}, \orgaddress{\street{9500 Gilman Dr}, \city{La Jolla}, \postcode{92093}, \state{CA}, \country{USA}}}
\affil[2]{\orgname{Shield AI}, \orgaddress{\street{600 W Broadway Suite \#600}, \city{San Diego}, \postcode{92101}, \state{CA}, \country{USA}}}

\abstract{
	Autonomous flight in cluttered environments requires a robot to build a geometric map of its surroundings and plan safe, dynamically feasible trajectories, all onboard and in real time. Conventional approaches treat mapping and planning as separate stages and often rely on binary occupancy for collision checking. We argue that these two stages should be co-designed around a single representation: a signed distance function (SDF). By encoding distance to the nearest obstacle, an SDF provides richer information for planning and trajectory optimization than occupancy alone. We develop an Octree REsidual Network (OREN) that pairs an explicit octree prior with an implicit neural residual to reconstruct SDFs online from point cloud observations with the efficiency of volumetric methods and the accuracy and differentiability of neural methods. In tandem, we develop Bubble$^\star$, a search-based planner that exploits the distance information to grow maximal collision-free balls, which we call bubbles, with formal guarantees of termination, completeness, and failure detection. Planning over a graph of bubbles significantly reduces collision checks compared to a grid-based A$^\star$ search and returns a bubble sequence that forms a safe corridor for trajectory optimization. We demonstrate the integrated OREN-Bubble$^\star$ approach onboard a quadrotor, navigating unseen indoor environments in real time under tight compute constraints. OREN improves SDF estimation by $22$\% compared to baselines, while Bubble$^\star$ finds trajectories spanning $\approx 90$~m through a cluttered environment in $1$--$3$~sec., whereas baselines take up to $10$~sec. in the same environment.
}


\maketitle

\clearpage
\section{Introduction}
\label{sec:introduction}

Autonomous flight in unknown environments requires an unmanned aerial vehicle (UAV) to build a map of its surroundings and to plan safe, dynamically feasible motion through it, all onboard and in real time from streaming sensor data. This capability underpins many applications of aerial robotics, including inspection, delivery, and search and rescue, where UAVs must navigate safely and efficiently.
Conventionally, mapping and planning are treated as two separate stages connected by a binary occupancy map. The planner queries the map to test for collisions at each robot configuration, disregarding distance and gradient information that a richer distance-based representation could provide. We argue that the two stages can instead be co-designed around a signed distance function (SDF) representation, exploiting the distance information to reduce the collision checking operations, accelerate the planning, and provide constraints for trajectory optimization.

Accurate environment representations are essential across the robot autonomy stack, including in simultaneous localization and mapping \citep{ortiz_isdf_2022,pan_pin-slam_2024,miso2025}, navigation \citep{oleynikova_voxblox_2017,long_sensor-based}, and manipulation~\citep{ReDSDF,li2024config,li2024robot, brunner_aerial_manipulation}.
SDF representations are particularly well suited to serve both mapping and planning. Given a query point, an SDF returns the signed distance to the nearest surface, with the sign indicating whether the query lies in free (positive) or occupied (negative) space.
It captures obstacle surfaces implicitly as its zero-level set \citep{park_deepsdf_2019}, while simultaneously providing clearance to the nearest obstacle.
To be useful in aerial autonomy, an SDF must be built and queried quickly from streaming observations, keeping a small memory footprint and remaining accurate under a tight onboard computational budget. However, existing SDF methods rarely meet all these requirements at once.

\begin{figure}[t]
    \centering
    \begin{subfigure}{\linewidth}
    \centering
    \includegraphics[width=\linewidth,trim={0 50pt 0 250pt},clip]
    {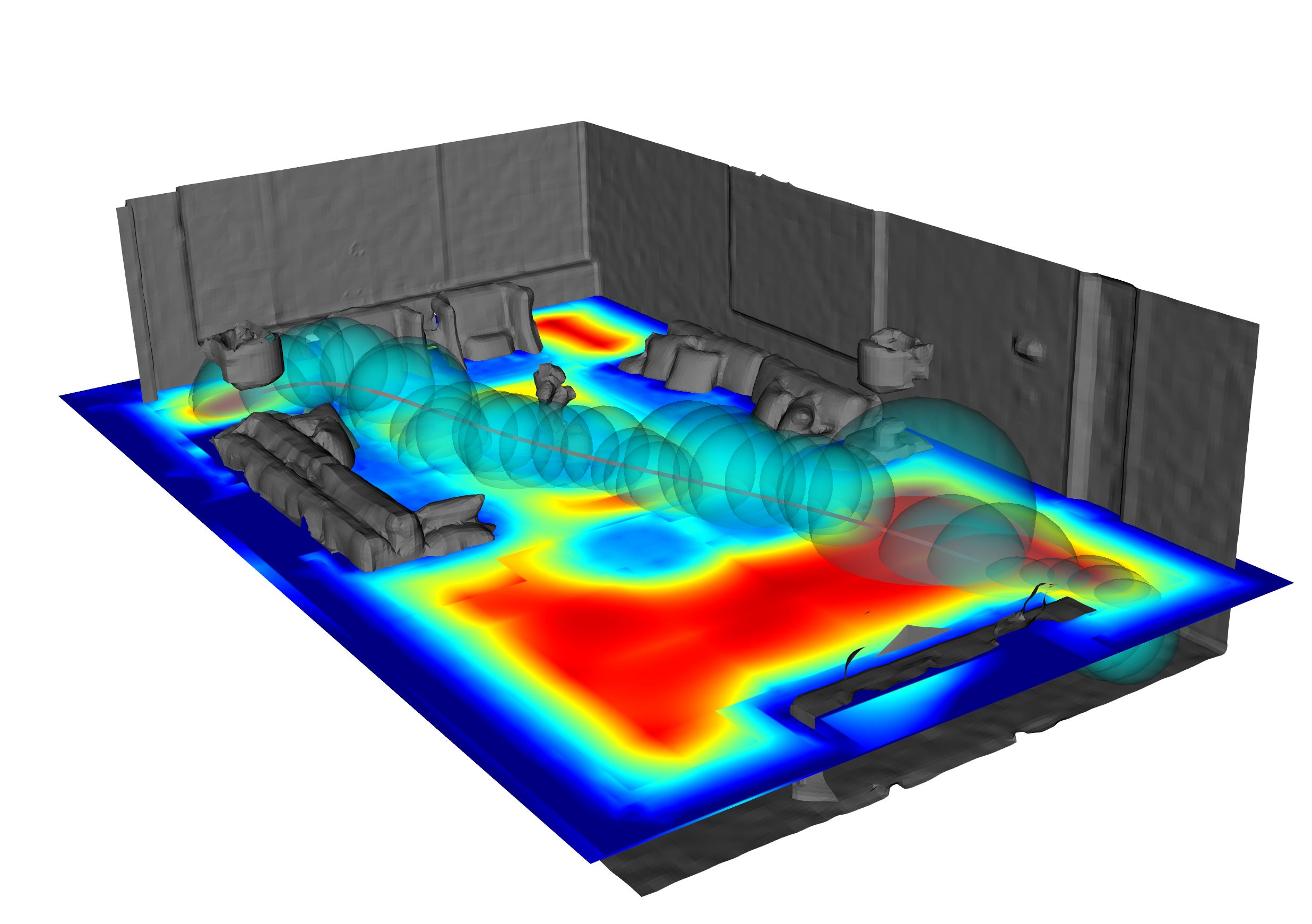}
    \caption{}
    \end{subfigure}
    \begin{subfigure}{\linewidth}
        \centering
        \includegraphics[width=\linewidth, trim=0cm 60cm 0cm 40cm, clip=true]{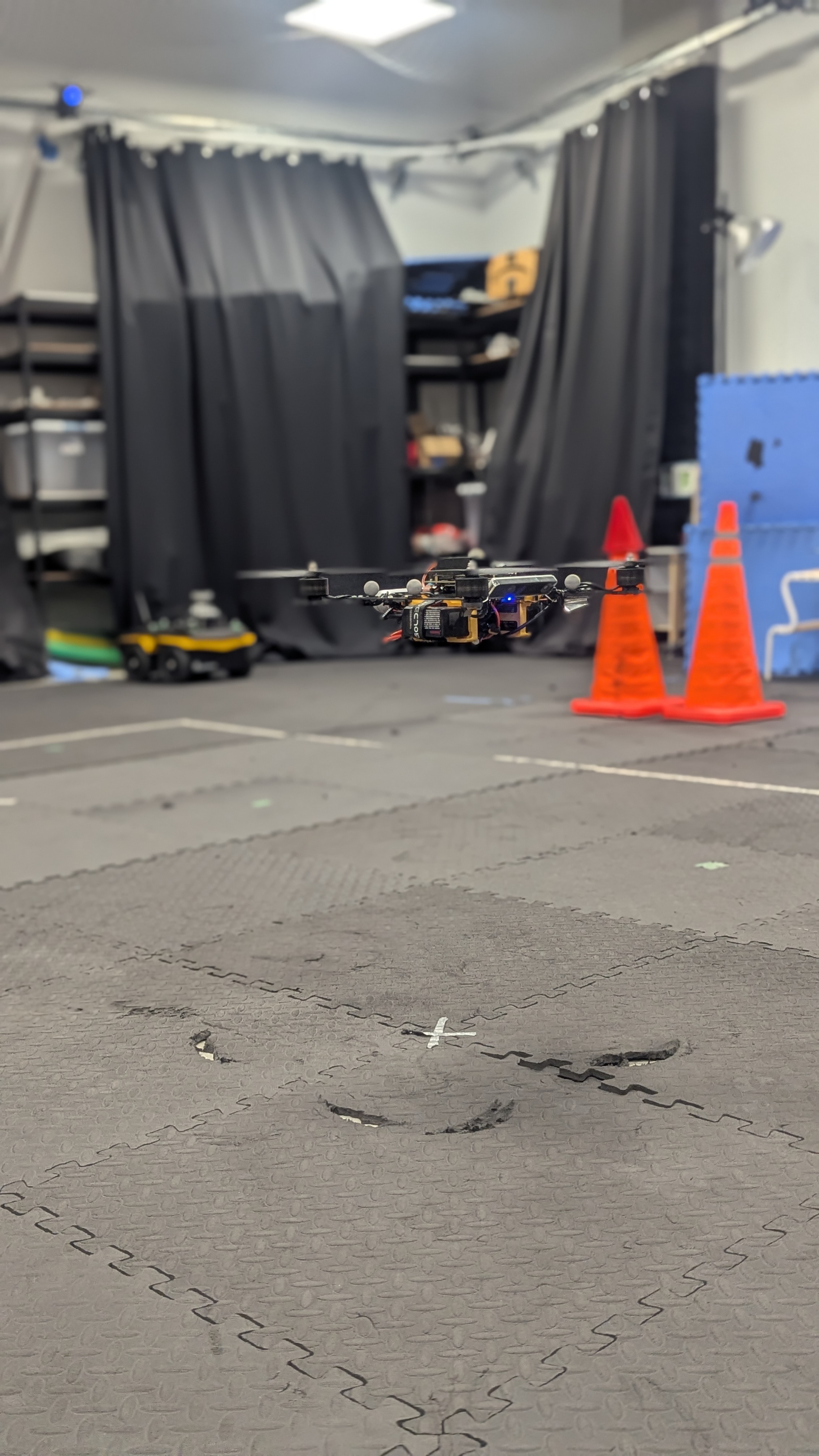}
        \caption{}
    \end{subfigure}
        \caption{\small We consider integrated mapping and planning for autonomous UAV flight. We develop \oren to reconstruct accurate Euclidean SDF online from streaming point clouds and \bubblestar to plan safe flight corridors exploiting the distance information. (a) The color map shows distance to the nearest surface ({\color{red}large} to {\color{blue}small}), and the planned safe corridor (composed of {\color{cyan}cyan balls}). (b) \oren and \bubblestar running onboard a quadrotor equipped with a Jetson Orin NX.}
        \label{fig:teaser}
\end{figure}

In this paper, we present a unified approach for mapping and planning for aerial robots, built around non-truncated SDF (Fig.~\ref{fig:teaser}). Our method reconstructs SDF from point cloud measurements onboard a UAV and enables efficient distance-aware motion planning and trajectory optimization with safety and dynamic-feasibility guarantees. We extend our prior conference paper, \citet{dai_oren_2025}, which introduced an Octree REsidual Network (\oren) for SDF mapping, combining an explicit prior from octree interpolation with implicit neural features decoded into a correction of the prior.
To construct the prior, we use an octree with distance and gradient estimates stored at the octant vertices and gradient-augmented interpolation to compute SDF at arbitrary query positions.
We correct the prior with a neural network residual, which recovers fine details of the observed surface from implicit features.

Given an SDF model of the environment, the second challenge is turning it into safe, dynamically feasible UAV trajectories.
Commonly, this is done by first planning a collision-free path, e.g., via A$^\star$ or RRT \citep{heuristic_basis, karaman_sampling}, then constructing a safe corridor of overlapping convex regions around the path \citep{liu_corridor}, and finally optimizing a dynamically feasible trajectory inside the corridor \citep{MINCO}. Bubble Planner \citep{ren2022bubbleplannerplanninghighspeed}, a representative work, finds a path via search, grows a receding corridor of spheres around it, and optimizes a trajectory through the corridor. 
We observe that with SDF, the path search and the corridor construction need not be separate stages because the distance at any point already defines a maximal collision-free ball around it, which we call a \emph{bubble}.
We develop \bubblestar, a motion planning algorithm that constructs a graph of bubbles from SDF queries to plan a safe bubble corridor. We establish formal guarantees of termination, completeness, and failure detection under mild clearance assumptions, then optimize a dynamically feasible trajectory within the corridor using MINCO \citep{MINCO}, a minimum-control-effort polynomial parameterization. Because \bubblestar reads signed distances directly from \oren, mapping and planning share a single representation, requiring far fewer collision checks than grid-based A$^\star$. In our experiments, \oren improves SDF estimation by $22$\% over baseline methods and runs efficiently enough for real-time deployment onboard a UAV, while \bubblestar finds trajectories $\approx 90$~m long through a cluttered environment in $1$--$3$~sec., versus up to $10$~sec. for existing methods.

In summary, our work makes the following contributions.
\begin{itemize}
    \item We develop \oren, a mapping approach that uses an explicit octree prior and implicit neural correction to reconstruct accurate, differentiable, non-truncated SDF in real time.

    \item We propose \bubblestar, a search-based planner that exploits an SDF map to construct a graph of collision-free spheres (bubbles), unifying path search and safe-corridor construction into a single algorithm. We provide formal guarantees of termination, completeness, and failure detection, and perform trajectory optimization within the bubble corridor to produce dynamically feasible UAV trajectories.

    \item We compare our methods to mapping and planning baselines and demonstrate the complete approach in real-world autonomous quadrotor flight in unknown environments.
\end{itemize}

\section{Related Work}
\label{sec:related_work}

This section reviews existing methods for SDF reconstruction and integrated mapping and planning for aerial robots.
\subsection{SDF Reconstruction}

Methods for learning SDFs fall into three broad groups: volumetric, Gaussian Process (GP), and neural network.

\emph{Volumetric methods} \citep{curless_volumetric_1996,newcombe_kinectfusion_2011,kahler_very_2015,oleynikova_voxblox_2017,han_fiesta_2019,pan_voxfield_2022,millane_nvblox_2024} fuse observations into a regular grid using voxel hashing for efficient updates and queries. Voxblox \citep{oleynikova_voxblox_2017} builds a TSDF layer from projective distance, then propagates it to an SDF layer via breadth-first search (BFS). Both steps introduce errors that subsequent works \citep{han_fiesta_2019,pan_voxfield_2022} reduce by using non-projective distance and replacing BFS path length with distance to the nearest oriented surface point. Nonetheless, these representations are discrete, non-differentiable, and hard to scale to large scenes.
In comparison, \oren learns a continuous, differentiable SDF in real time, with low memory usage that scales to large scenes. \oren first estimates an SDF prior from explicit discrete priors stored in an octree, then uses implicit neural features to predict residuals that correct the prior, forming a compact, differentiable representation of continuous SDF.

\emph{Gaussian Process (GP) methods} \citep{lee2019gpis,lan2021loggpis,wu_vdb-gpdf_2025,dai2026kernelsdf} learn continuous fields supporting gradient computation and uncertainty quantification. GP implicit surface (GPIS) \citep{lee2019gpis} iteratively estimates oriented surface points and regresses SDF online, learning the surface implicitly as the SDF's zero-level set. GPIS is accurate near the surface but fails to extrapolate away from it, since training data comes only from near-surface points and distant queries fall back to the zero-mean prior.
Based on the connection between the heat equation and the unsigned distance function (UDF) \citep{varadhan_behavior_1967}, Log-GPIS \citep{lan2021loggpis} learns globally-generalized unsigned distances in log space. VDB-GPDF \citep{wu_vdb-gpdf_2025} extends this log-GP technique to jointly learn surface estimation and UDF, using OpenVDB \citep{openvdb25} for memory efficiency.
However, both methods omit the sign and struggle to scale to large scenes on resource-constrained platforms due to the cubic complexity of matrix inversion during training.
In comparison, \oren computes the prior via $O(1)$ trilinear interpolation in the octree and the residual via matrix multiplication, roughly $O(n^2)$ for $n$ hidden dimensions, while using less memory than GP.

\emph{Neural network methods} are attractive for their native speed on GPUs. DeepSDF \citep{park_deepsdf_2019} showed that neural networks can learn continuous implicit SDFs, inspiring many follow-up works. iSDF \citep{ortiz_isdf_2022} learns SDF incrementally with Eikonal regularization. NeuS \citep{wang_neus_2021} jointly learns SDF and radiance fields \citep{nerf2020}. \citet{gropp_implicit_2020,takikawa_lod_2021} refine the network architecture and loss designs. These methods learn accurate SDF near the surface, which suffices for surface reconstruction, but rarely far from it. HotSpot \citep{wang_hotspot_2024} learns non-truncated SDF but is only verified at the object level and requires extensive data and training.

Recent work develops promising hybrid models that combine an explicit geometric structure with implicit neural features. PIN-SLAM \citep{pan_pin-slam_2024} stores neural features in near-surface voxels and decodes the SDF from nearby features, while \htwomapping \citep{jiang_h2-mapping_2023} combines an octree-based SDF prior with a neural residual. However, both methods learn only truncated SDF. HIO-SDF \citep{hio-sdf_2024} removes truncation by training on global priors from Voxfield \citep{pan_voxfield_2022}, but inherits the volumetric method's limited accuracy, and over-smoothes as the scene grows due to fixed network capacity.

In contrast, our method, \oren, builds an extendable semi-sparse octree that stores SDF values and gradients to efficiently capture the SDF of the whole space as the environment grows. Through gradient-augmented interpolation, \oren produces more accurate SDF priors, leaving more capacity for the subsequent neural network to recover surface detail from implicit features. Furthermore, the loss function is designed to encourage the network to learn accurate SDF far from the surface, which is critical for planning, and enables faster convergence than existing methods, allowing \oren to run in real time on a UAV.

\subsection{Integrated Mapping and Planning for Aerial Robots}

Motion planning for quadrotors involves local trajectory optimization and, most often, global path planning, both relying on \emph{differential flatness} \citep{MinSnapTrajectory,faessler2018differential}: the property that a dynamical system's full states and controls can be computed analytically from a reduced state, the \emph{flat output}, and its derivatives up to the relative degree.

Quadrotor dynamics are differentially flat, with 3D position and yaw as flat outputs of relative degree four: all quadrotor states (e.g., orientation) and controls (e.g., thrust) can be computed analytically from up to the fourth derivatives of position and yaw \citep{MinSnapTrajectory}, even with drag \citep{faessler2018differential}.
Thus, planning methods need only produce a position trajectory that is at least four times differentiable, and optionally yaw, without considering the full dynamics.

\emph{Mapless local trajectory optimization methods} plan directly from observations to avoid obstacles, trading global optimality for fast computation. \citet{song_rl_drone_racing} achieve high-speed drone racing by using reinforcement learning to generate trajectories from a single image, treating planning as a black box. \citet{jacquet2025neural} take a grey-box approach, learning a network that converts a depth image to a local SDF for local nonlinear model-predictive control (NMPC). All these methods share a limited field of view, risking dead-end failures. \citet{mapless_planner} partially mitigate this by efficiently storing sensor history.
Since they do not rely on a map, these methods often lack formal guarantees of optimality, leading to local optima.

\emph{Global planning methods} store obstacle information in a map and plan safe paths against it.
Global motion planning is dominated by search-based algorithms, e.g., A$^\star$ \citep{heuristic_basis}, and sampling-based ones, e.g., RRT \citep{lavalle2001rrt} and its asymptotically optimal variant RRT$^\star$ \citep{karaman_sampling}.
Sampling-based methods are widely used, e.g., in OMPL \citep{ompl_2012}, with quadrotor examples by \citet{gao2019flying,funk2021multires}.
Because they rely on random sampling, their completeness and optimality guarantees are only probabilistic or asymptotic.
In contrast, search-based methods operate on regular grids or state lattices \citep{likhachev2009lattice}, guaranteeing deterministic (resolution-)completeness or optimality, albeit confined by the chosen resolution.
\citet{liu_corridor, dharmadhikari2020motionprimitives} use search-based methods for quadrotor navigation.
Our approach mitigates this cost-versus-speed trade-off in resolution choice by building safe bubbles from clearance information, obviating collision checking within them.

Both search-based and sampling-based planners can be either \emph{geometric} or \emph{kinodynamic}.
Geometric planners produce a sequence of waypoints, assuming any pair of waypoints is achievable, whereas kinodynamic planners account for kinematic or dynamic constraints by requiring a valid motion between waypoints.
To plan dynamically feasible quadrotor trajectories, one can either 1) use a geometric planner followed by trajectory optimization to track the resulting waypoints, or 2) use a single kinodynamic planner directly.

Kinodynamic planning may appear simpler, as it uses a single planner. \citet{lavalle2001rrt} originally designed RRT for kinodynamic planning. However, \citet{Richter2016} show that for real quadrotors, full kinodynamic planning performs worse than combining geometric planning with trajectory optimization, due to its additional computation. \citet{allen2019realtime} speed up real-time sampling-based kinodynamic planning by introducing a learned reachability classifier between states. Meanwhile, \citet{mueller_motion_primitive} present efficient methods for computing \emph{motion primitives}: a library of trajectories used to connect states during planning. \citet{liu2017search} use such motion primitives in a search-based kinodynamic planner, showing they discretize the state space into a lattice suitable for efficient search. \citet{zhou_robust_efficient} improve search-based planning with continuous optimization, and \citet{ryll_efficient_traj_planning} present a receding-horizon variant for unknown environments. To handle state uncertainty alongside dynamic feasibility, \citet{agha2011firm,majumdar2017funnel} plan between distributions rather than states, composing feedback controllers to reach one distribution from another.
Our approach is most similar to \citet{Richter2016}, combining a geometric planner with subsequent trajectory optimization.
However, our geometric planner composes collision-free regions, leaving room for dynamic feasibility while ensuring collision avoidance.

\emph{Safe corridor methods}, such as \citet{Deits2015, liu_corridor, ren2022bubbleplannerplanninghighspeed, MINCO, tordesillas2022faster, lee2024safebubblecovermotion}, similarly find a sequence of collision-free regions (safe corridors).
These methods are popular for quadrotor navigation because they greatly accelerate trajectory optimization, making them suitable for onboard and even GPS-denied navigation~\citep{mohta_gps_denied}.
Most safe corridor methods first plan a geometric path, then expand convex regions around the waypoints.
\citet{liu_corridor} first introduce safe corridors: overlapping collision-free polyhedra expanded around a path found via jump point search \citep{harabor2011jumppointsearch} on an OctoMap \citep{hornung13auro}, with an optimal continuous trajectory found via quadratic programming (QP).
\citet{MINCO} accelerate optimization by introducing a closed-form trajectory parameterization enabling unconstrained optimization within safe corridors.
Subsequent approaches \citep{ren2025safety,ren2022bubbleplannerplanninghighspeed} speed computation using lighter local maps instead of an OctoMap.
In particular, \citet{ren2022bubbleplannerplanninghighspeed} build an efficient KD-Tree for proximity queries and use clearance information to build spherical bubbles along a geometric path.
Our approach also builds spherical bubbles from clearance but directly searches for a sequence of them.
In contrast to inflating corridors around a pre-computed path, few recent methods directly search for a sequence of corridors.
Graph-of-convex-sets formulations \citep{Marcucci_motion_planning,marcucci2024boxes} jointly optimize corridor choice and the trajectory within them via mixed-integer optimization.
However, these methods require a known corridor set, and mixed-integer optimization can be expensive.
In aerial robotics, \citet{gao2019flying} search for a sequence of spherical corridors using a modified RRT$^\star$, enabling high-speed LiDAR-based navigation. Similarly, \citet{funk2021multires} find convex regions via RRT search with collision-checking against a multi-resolution SDF grid built from visual sensors \citep{laina2025scalable}. Our prior work, \citet{lee2024safebubblecovermotion}, used a sampling-based formulation that first covers free space with a graph of bubbles and searches it via Dijkstra's shortest-path algorithm. Here, we use a search-based formulation that improves efficiency by interleaving the graph construction and the search and by biasing the search towards the goal via a heuristic function, similarly to A$^\star$.

\begin{figure*}[t]
	\centering
	\includegraphics[width=\linewidth]{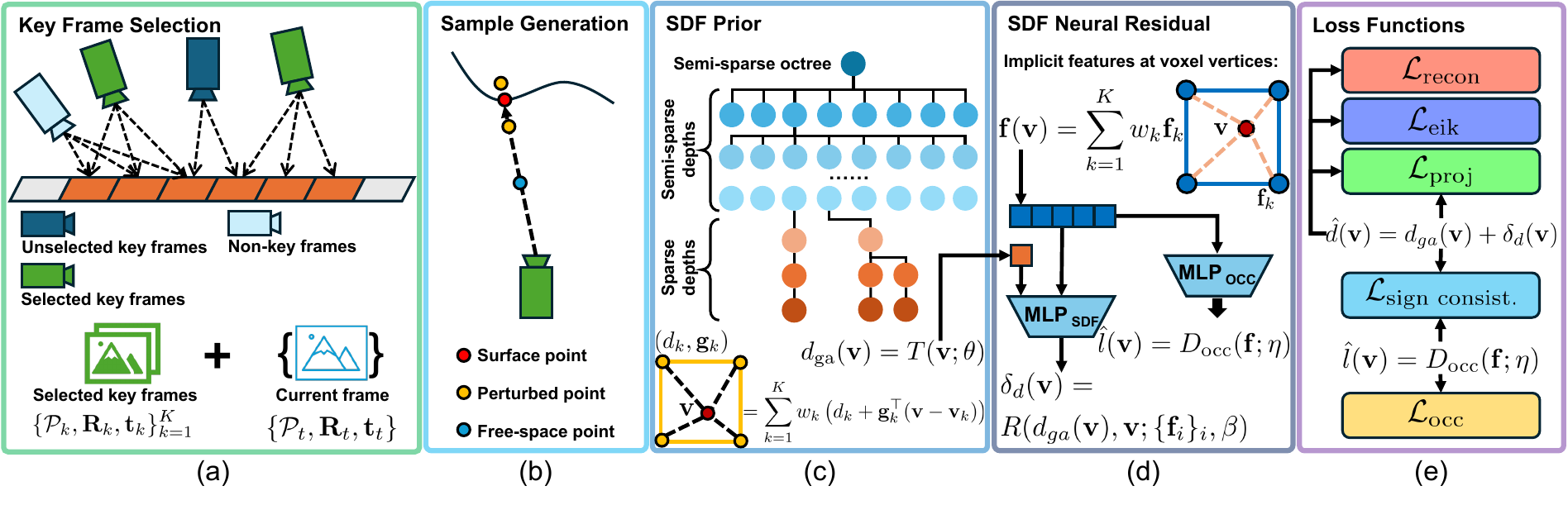}
	\caption{\small Overview of \oren: a) We keep key frames with small overlap and those that maximize the surface coverage for training; b) with the selected key frames and the current frame, we generate three types of samples: \textcolor{red}{surface} points, \textcolor{orange}{perturbed} points around the surface, and \textcolor{cyan}{free-space} points; c) to predict SDF, we first obtain an SDF prior $d_{ga}(\bfv)$ with gradient-augmented interpolation in a semi-sparse octree, where each octant vertex has estimated SDF value and gradient; d) we also obtain an implicit neural feature for $\bfv$ by trilinear interpolation of implicit features stored at the voxel's vertices, which is fed into an MLP decoder to obtain an SDF residual correction $\delta_d(\bfv)$, and another MLP decoder to predict occupancy; e) the SDF prior $d_{ga}(\bfv)$ and the SDF residual $\delta_d(\bfv)$ are combined as the final SDF prediction $\hat{d}(\bfv) = d_{ga}(\bfv) + \delta_d(\bfv)$, and the parameters are trained with five loss functions: \textcolor{red}{reconstruction} loss, \textcolor{violet}{Eikonal} loss, \textcolor{green}{projection} loss, \textcolor{orange}{occupancy} loss and \textcolor{cyan}{sign consistency} loss.}
	\label{fig:method_overview}
\end{figure*}

\section{Problem Statement}
\label{sec:problem_statement}
Consider a 3D environment $\calR \subset \bbR^3$ with obstacles represented as a subset $\Omega \subset \calR$. The SDF $d: \bbR^3 \to \bbR$ of $\Omega$ is defined as the shortest distance from any point $\bfv \in \bbR^3$ to the obstacle surface $\partial \Omega$, with a sign indicating whether $\bfv$ is inside or outside of $\Omega$:
\begin{equation} \label{eq:sdf_definition}
	d(\bfv) =
	\begin{cases}
		\phantom{+}\min_{\bfy \in \partial \Omega} \left\|\bfv - \bfy\right\|_2, & \bfv \not\in \Omega, \\
		-\min_{\bfy \in \partial \Omega} \left\|\bfv - \bfy\right\|_2,           & \bfv \in \Omega.
	\end{cases}
\end{equation}
The SDF satisfies two key properties: 1) the obstacle surface is encoded as the zero-level set, $d(\bfv) = 0$, $\forall \bfv \in \partial\Omega$; and 2) the gradient of $d(\bfv)$ is the unit vector pointing away from the nearest surface point and satisfies an Eikonal equation~\citep{ortiz_isdf_2022}:
\begin{equation}
	\nabla d(\bfv)=\frac{\bfv-\bfv_*}{d(\bfv)}, \quad \left\|\nabla d(\bfv)\right\|_2=1, \; \text{a.e.},\label{eq:sdf_constraints}
\end{equation}
where $\bfv_* \in \arg \min_{\bfy\in\partial\Omega} \left\|\bfv-\bfy\right\|_2$.

We consider a quadrotor robot, equipped with a range sensor (e.g., LiDAR or depth camera), operating in the environment. Given a stream of point clouds $\calP_t$, our objective is to: 1) obtain an estimate $\hat{d}: \bbR^3 \to \bbR$ of the SDF of $\Omega$ and 2) plan a safe dynamically feasible trajectory for the robot. Let $\bfx = (\bfp, \bfnu, R, \bfomega) \in \calX := \bbR^3 \times \bbR^3 \times SO(3) \times \bbR^3$ be the quadrotor state, consisting of its position $\bfp$ and velocity $\bfnu$ in the inertial frame, orientation $R \in SO(3)$, and body angular velocity $\bfomega$. Given control input $\bfu = (F, \bftau) \in \bbR \times \bbR^3$, including the collective motor thrust $F$ and body torque $\bftau$, the quadrotor dynamics are:
\begin{equation} \label{eq:quad_dynamics}
	\dot{\bfx} \!=\! \bff(\bfx) + \bfG(\bfx) \bfu =\!
	\begin{cases}
		\dot{\bfp} = \bfnu,                                 \\
		\ddot{\bfp} = -g\, \bfe_3 + \frac{1}{m} R F \bfe_3,\!\!\! \\
		\dot{R} = R\, \hat{\bfomega},                      \\
		\dot{\bfomega} = J^{-1} (\bftau - \hat{\bfomega} J \bfomega),
	\end{cases}
\end{equation}
where $m$ is the mass, $J$ is the inertia matrix, $\bfe_3 = (0, 0, 1)^\top$, $g$ is the gravitational acceleration, and $\hat{\bfomega}$ is a skew-symmetric matrix formed from $\bfomega$.

The quadrotor dynamics \eqref{eq:quad_dynamics} are differentially flat \citep{MinSnapTrajectory, MINCO}; there exists a flat output $\bfz$ such that the state and control can be expressed as algebraic functions of $\bfz$ and $k$ of its derivatives,
\begin{equation*}
	\bfx = \Psi_{\bfx}(\bfz, \dot{\bfz}, \ldots, \bfz^{(k-1)}), \;
	\bfu = \Psi_{\bfu}(\bfz, \dot{\bfz}, \ldots, \bfz^{(k)}).
\end{equation*}
For a quadrotor, the flat output $\bfz = (\bfp, \psi)$ consists of its position $\bfp$ and yaw angle $\psi$. The state $\bfx$ and control $\bfu$ can be recovered from derivatives of the flat output up to order $k=4$ with $\Psi_{\bfx}$ and $\Psi_{\bfu}$ known in closed form \citep{MinSnapTrajectory}. The yaw $\psi$ can be determined separately, e.g., along the velocity direction, so motion planning reduces to designing a continuous position trajectory $\bfp: [0, T] \rightarrow \bbR^3$ from a start $\bfp_s$ to a goal $\bfp_g$ that remains collision-free and respects the vehicle dynamics.
We express this as a trajectory optimization problem:
\begin{equation} \label{eq:planning_problem}
	\begin{aligned}
		\min_{\bfp(\cdot),\, T} \quad & \int_0^T \| \bfp^{(4)}(t) \|_2^2 \, dt + \rho\, T                 \\
		\text{s.t.} \quad             & \bfp(0) = \bfp_s, \quad \bfp(T) = \bfp_g,                                    \\
		                              & \hat{d}(\bfp(t)) \geq r, \quad \forall t \in [0, T],                         \\
		                              & \left\| \dot{\bfp}(t) \right\|_2 \leq v_{\max},  \quad \forall t \in [0, T], \\
		                              & \left\| \ddot{\bfp}(t) \right\|_2 \leq a_{\max}, \quad \forall t \in [0, T],
	\end{aligned}
\end{equation}
where the objective minimizes the trajectory snap $\|\bfp^{(4)}\|_2^2$ with $\rho > 0$ trading off smoothness against duration. In the constraints, $r > 0$ is a safety radius accounting for the robot size and errors in the SDF estimate $\hat{d}$, while $v_{\max}$ and $a_{\max}$ are velocity and acceleration bounds. 

In summary, we consider an integrated mapping and planning problem. Given point cloud measurements, we estimate the SDF $\hat{d}$ and plan a dynamically feasible trajectory $\bfp(t)$ from the current robot position $\bfp_s$ to a desired goal $\bfp_g$ that maintains clearance $\hat{d}(\bfp(t)) \geq r$ for all $t \in [0,T]$.

\section{Octree Residual Network for SDF Mapping}
\label{sec:grad_sdf}

We first focus on reconstructing the SDF in \eqref{eq:sdf_definition} from streaming point cloud measurements.
We develop \oren, a hybrid model that combines an explicit octree prior with an implicit neural correction. 
We present an overview of \oren in Fig.~\ref{fig:method_overview}.
An octree data structure stores explicit SDF and gradient estimates, from which a coarse SDF prior is obtained by gradient-augmented interpolation (Sec.~\ref{sec:sdf_prior}). To recover the geometric details that the octree resolution cannot capture, implicit neural features are stored at the octant vertices and are decoded by an MLP into a residual correction of the prior (Sec.~\ref{sec:sdf_residual}). We also decode the implicit features with a second MLP to predict occupancy, which is used to supervise the SDF sign during online training and improve the robustness to sensor noise. The resulting non-truncated SDF estimate $\hat{d}$ can subsequently be used in the UAV trajectory optimization problem in \eqref{eq:planning_problem}.

\subsection{SDF Prior From Octree Interpolation}
\label{sec:sdf_prior}

\oren computes the SDF prior through interpolation of SDF values and gradients stored in an octree with \emph{sparse} and \emph{semi-sparse} layers.
This allows efficient storage compared to a dense, regular grid.
Of the $N$ layers of our octree, the first $M$ layers are designed to be \emph{semi-sparse} in the sense that all siblings of an occupied child octant are created regardless of occupancy. The remaining $N-M$ layers are \emph{sparse}, where only child octants containing surface points are populated. 
This is illustrated in Fig.~\hyperref[fig:method_overview]{2c}.
Each octant vertex $\bfv_k$, with $k \in \{1,\ldots,8\}$, stores learnable estimates $d_k \in \bbR$ and $\bfg_k \in \bbR^3$ of the SDF $d(\bfv_k)$ and its gradient $\nabla d(\bfv_k)$, and vertices are shared across neighboring octants at different tree depths to save memory. The semi-sparse layers cost extra memory but yield a more accurate prior, especially for query positions away from the surface, because creating sibling octants places vertices closer to an arbitrary query $\bfv$ and reduces the interpolation discontinuities at octant boundaries that arise in a purely sparse octree \citep{dai_oren_2025}. 
Using a semi-sparse octree of resolution $\ell$, for a query near the surface we can locate an octant no larger than $\ell \times 2^{N-M}$, while for distant queries a large empty octant suffices for an accurate prior under the gradient-augmented interpolation described next.

\begin{figure*}[t]
	\centering
	\begin{subfigure}[t]{0.16\linewidth}
		\includegraphics[width=\linewidth]{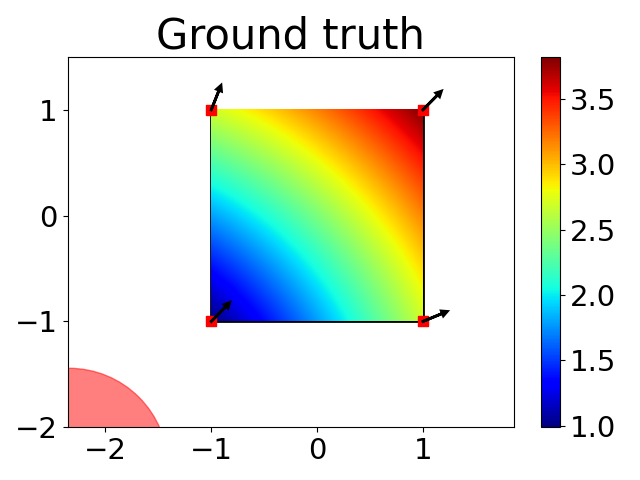}
	\end{subfigure}%
	\hfill%
	\begin{subfigure}[t]{0.16\linewidth}
		\includegraphics[width=\linewidth]{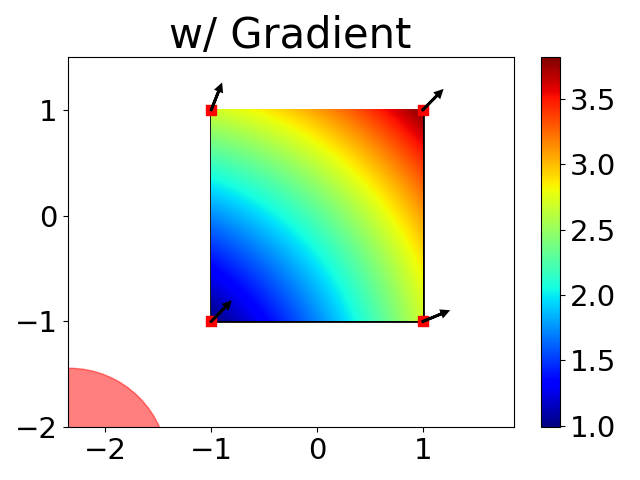}
	\end{subfigure}%
	\hfill%
	\begin{subfigure}[t]{0.16\linewidth}
		\includegraphics[width=\linewidth]{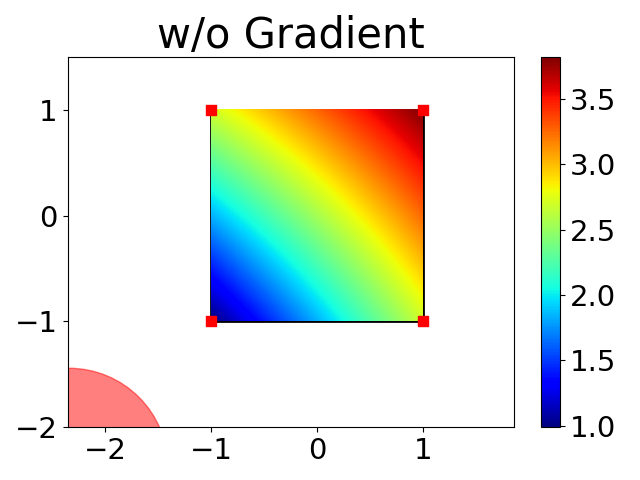}
	\end{subfigure}%
	\hfill%
	\begin{subfigure}[t]{0.16\linewidth}
		\includegraphics[width=\linewidth]{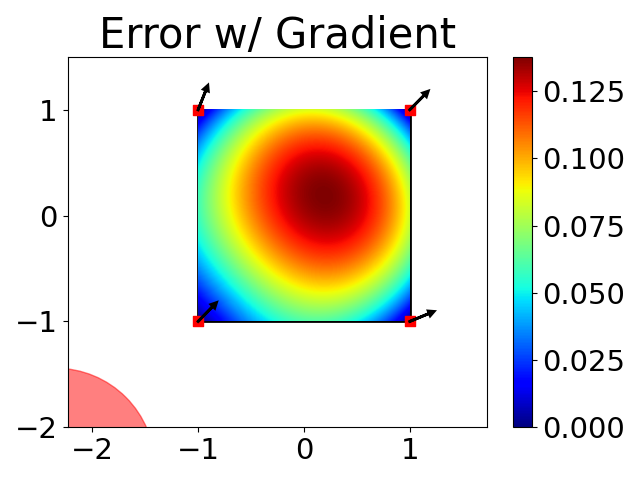}
	\end{subfigure}%
	\hfill%
	\begin{subfigure}[t]{0.16\linewidth}
		\includegraphics[width=\linewidth]{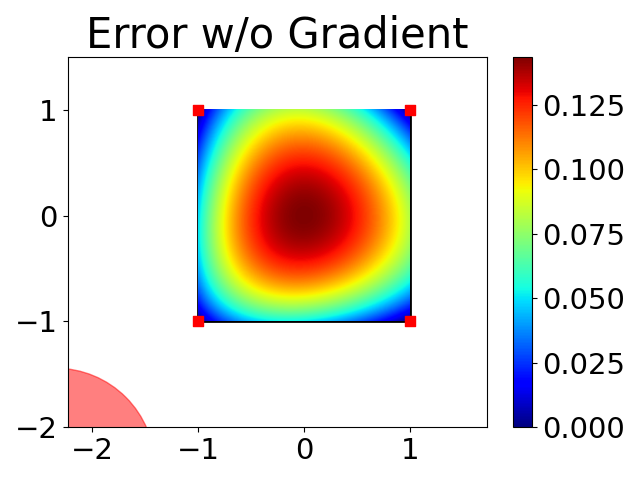}
	\end{subfigure}%
	\hfill%
	\begin{subfigure}[t]{0.16\linewidth}
		\includegraphics[width=\linewidth]{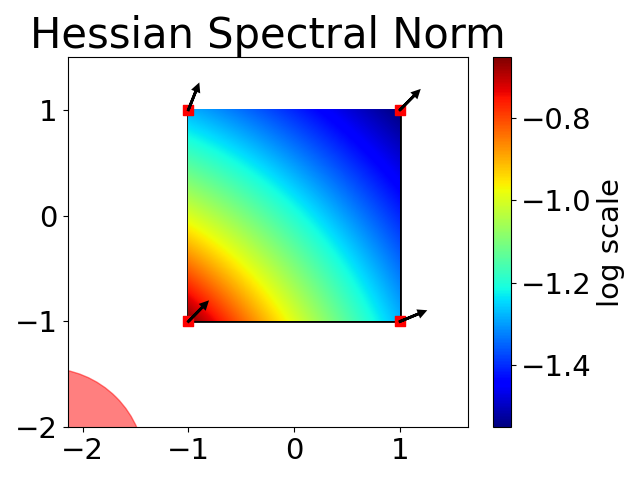}
	\end{subfigure}
	\begin{subfigure}[t]{0.16\linewidth}
		\includegraphics[width=\linewidth]{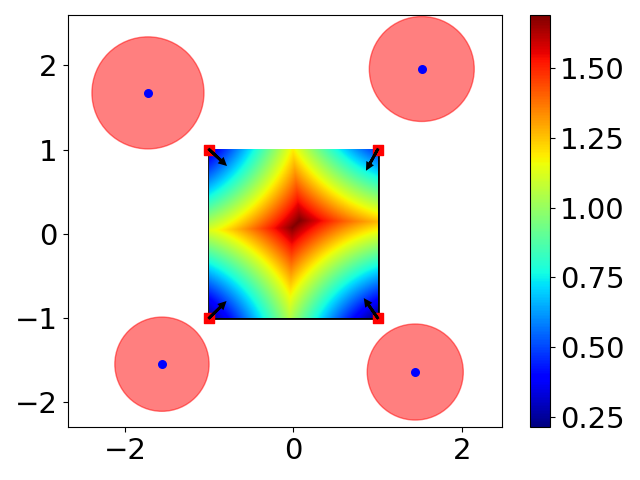}
		\caption{}
		\label{fig:comp_interpolation_a}
	\end{subfigure}%
	\hfill%
	\begin{subfigure}[t]{0.16\linewidth}
		\includegraphics[width=\linewidth]{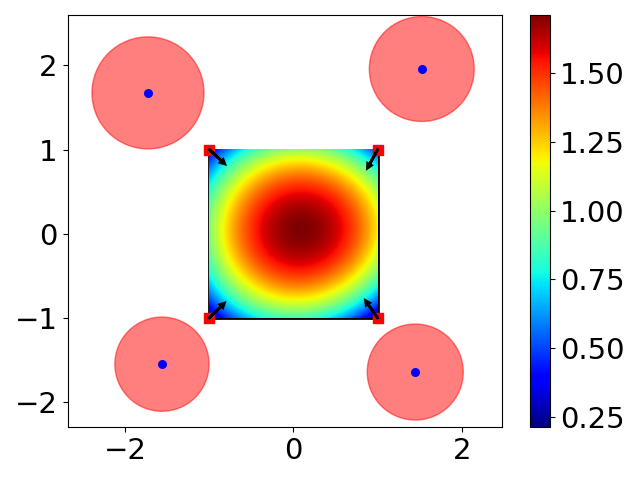}
		\caption{}
		\label{fig:comp_interpolation_b}
	\end{subfigure}%
	\hfill%
	\begin{subfigure}[t]{0.16\linewidth}
		\includegraphics[width=\linewidth]{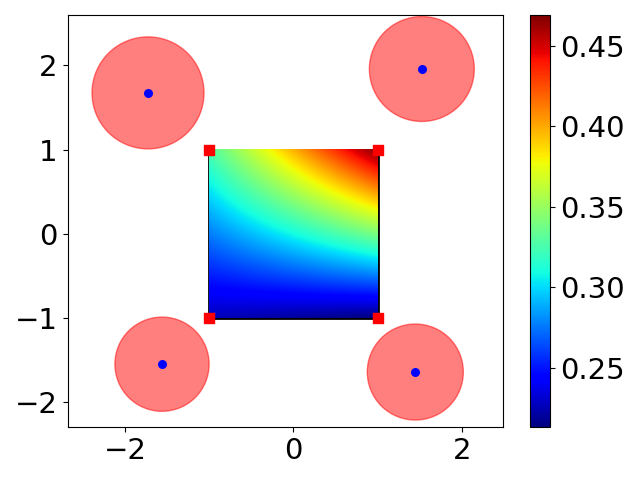}
		\caption{}
		\label{fig:comp_interpolation_c}
	\end{subfigure}%
	\hfill%
	\begin{subfigure}[t]{0.16\linewidth}
		\includegraphics[width=\linewidth]{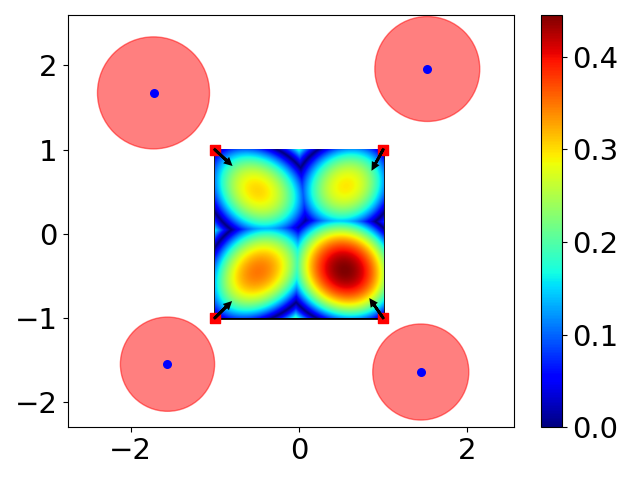}
		\caption{}
		\label{fig:comp_interpolation_d}
	\end{subfigure}%
	\hfill%
	\begin{subfigure}[t]{0.16\linewidth}
		\includegraphics[width=\linewidth]{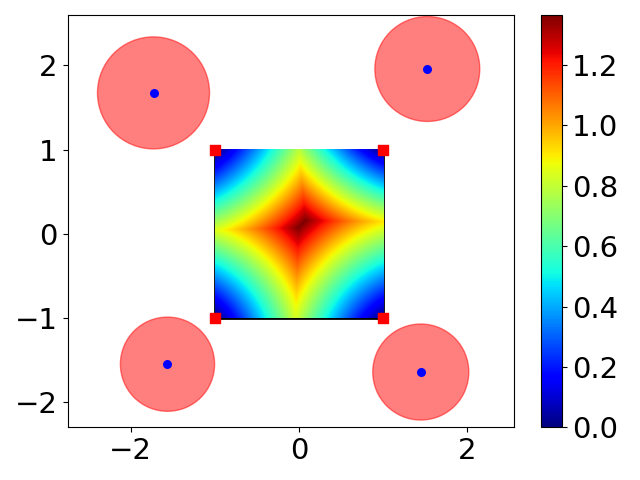}
		\caption{}
		\label{fig:comp_interpolation_e}
	\end{subfigure}%
	\hfill%
	\begin{subfigure}[t]{0.16\linewidth}
		\includegraphics[width=\linewidth]{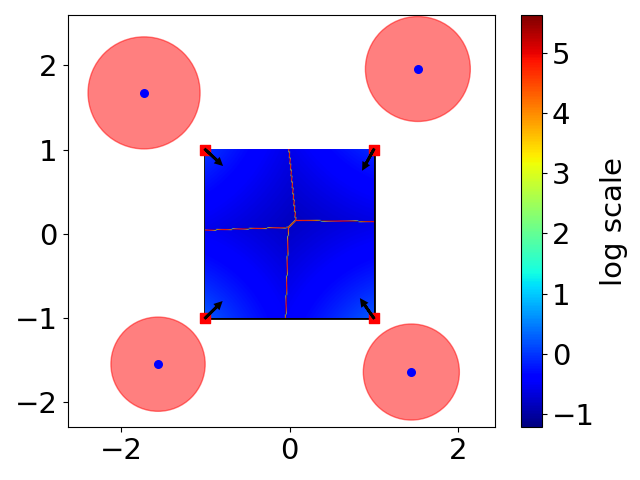}
		\caption{}
		\label{fig:comp_interpolation_f}
	\end{subfigure}
	\caption{\small 2D visualization of interpolation with and without gradient augmentation for one (red region, top row) and four obstacles (red regions, bottom row). Gradient-augmented interpolation produces a better SDF prior (b) with smaller error (d). Empirically, positions where the SDF gradient is not well defined (large Hessian spectral norm), as shown in (f), have small interpolation error with gradient augmentation as shown in (d).}
	\label{fig:comp_interpolation}
\end{figure*}

To produce an accurate enough prior for the residual network to only capture fine details, we use gradient-augmented trilinear interpolation. At the smallest octant containing a query position $\bfv$, we first extrapolate from each vertex $\bfv_k$:
\begin{equation}
	d_k(\bfv) = d_k + \bfg_k^\top (\bfv - \bfv_k),\ k\in\{1,\dots,8\},
\end{equation}
and combine the extrapolations into the gradient-augmented ($ga$) interpolation
\begin{equation} \label{eq:sdf_prior}
	d_{ga}(\bfv) = \frac{1}{\gamma} \sum_{k=1}^8 w_k d_k(\bfv),\ \gamma=\sum_{k=1}^8 w_k,
\end{equation}
where $w_k = 1/|\diag{(\bfv - \bfv_{k})}|$ is the trilinear interpolation weight. Unlike regular trilinear interpolation, which ignores the stored gradients, the gradient-augmented form admits a tighter error upper bound; we present the derivation in \citet{dai_oren_2025}.

Fig.~\ref{fig:comp_interpolation} illustrates the benefit of gradient-augmented interpolation in 2D scenes with one and four obstacles. Each row shows the ground-truth SDF (a), the interpolation results with (b) and without (c) gradient augmentation, the corresponding errors (d, e), and the Hessian spectral norm of the SDF (f). Gradient-augmented interpolation produces smaller errors, as seen in Fig.~\ref{fig:comp_interpolation_d} and \ref{fig:comp_interpolation_e}, and the improvement grows with the number of obstacles. Although the SDF gradient is not well defined on the medial axes, where the Hessian spectral norm is large (Fig.~\ref{fig:comp_interpolation_f}), gradient-augmented interpolation still attains small error there in practice (Fig.~\ref{fig:comp_interpolation_d}).

The prior $d_{ga}(\bfv)$ is thus computed from a semi-sparse octree with learnable SDF and gradient estimates $d_k$ and $\bfg_k$ 
at each vertex, which are
optimized jointly with the residual network. In the experiments, we use the octree configuration of \citet{dai_oren_2025}, with $N=8$ total octree layers with $M=5$ semi-sparse layers, and a resolution of $\ell=10$\,cm.

\subsection{SDF Residual From Implicit Feature Decoding}
\label{sec:sdf_residual}

The accuracy of the SDF prior is limited by the octree resolution, so it lacks geometric detail. To achieve high fidelity, we learn a residual correction with a neural network $R(d_{ga}(\bfv), \bfv; \{\bff_k\}_k, \beta)$ that composes octree feature interpolation $\bff(\bfv) = \sum_k w_k \bff_k$, using implicit neural features $\bff_k \in \bbR^F$ stored at the octree vertices, with an MLP decoder $D_\text{SDF}(d, \bff(\bfv); \beta)$. Each octant vertex is assigned a feature $\bff_k$, initialized to zero and optimized together with the decoder weights $\beta$. Octree expansion automatically allocates more features to the near-surface regions as smaller octants are created, enabling continual learning as the sensor moves. As shown in Fig.~\hyperref[fig:method_overview]{2d}, for a query point $\bfv$ we locate the leaf octant containing it, interpolate the feature $\bff(\bfv) = \sum_{k=1}^8 w_k \bff_k$ with the same weights $w_k$ as in \eqref{eq:sdf_prior}, and decode the SDF residual $\delta_d(\bfv) = D_\text{SDF}(d_{ga}(\bfv), \bff(\bfv); \beta)$. The final SDF prediction combines the prior and the residual,
\begin{equation}
	\hat{d}(\bfv) = d_{ga}(\bfv) + \delta_d(\bfv).
\end{equation}
In our experiments $F=3$ and the MLP has two 32-dimensional hidden layers with LeakyReLU activations.

Extending \citet{dai_oren_2025}, we add a second decoder MLP $D_\text{occ}(\bff(\bfv);\eta)$ to predict the occupancy log-odds $\hat{l}(\bfv)\in\bbR$, a continuous value that can be converted to an occupancy probability as $\sigma(\hat{l}(\bfv))$, where $\sigma(l) := (1 + e^{-l})^{-1}$ is the sigmoid function. Since $\sigma(l) > \frac{1}{2}$ if and only if $l > 0$, a positive prediction $\hat{l}(\bfv)>0$ indicates $\bfv\in\Omega$, and $\hat{l}(\bfv)<0$ indicates $\bfv\not\in\Omega$. Unlike the SDF branch, which decodes a residual correction to the octree prior, the occupancy decoder predicts $\hat{l}(\bfv)$ directly from the interpolated features $\bff(\bfv)$, without a prior.
The occupancy prediction $\hat{l}(\bfv)$ is used to supervise the SDF sign as a regularization term, which significantly improves the robustness to noise. The occupancy labels are obtained by ray-casting the depth measurements. Samples along the ray between the sensor and the measured surface, including samples perturbed to lie in front of the surface, are labeled free ($o = 0$), while samples perturbed to lie behind the surface are labeled occupied ($o = 1$). The sample at the measured surface receives $o = 0.5$, which places the decision boundary $\sigma(\hat{l}) = \frac{1}{2}$ on the surface, aligned with the zero level-set of the SDF.

We jointly train the octree parameters $\theta = \{d_k, \bfg_k, \bff_k\}_k$ and the decoder weights $\beta$ and $\eta$ online following \citet{dai_oren_2025}, 
with the extension of adding loss terms $\calL_\text{occ}$ and $\calL_\text{sign consist.}$ for learning occupancy and enforcing its consistency with the predicted SDF sign:
\begin{align}
\calL_\text{occ}(\hat{l},o)&=\operatorname{BCE}\left(\sigma(\hat{l}), o\right),\\
\calL_\text{sign consist.}(\hat{l}, \hat{d})&=\mathbbm{1}_{|\hat{l}|>\tau}\max\left(0, s\hat{d}\right),
\end{align}
where $\operatorname{BCE}(\cdot,\cdot)$ is the binary cross-entropy loss, $o\in\{0,1\}$ is the ground truth of occupancy (0: free, 1: occupied), $s=\operatorname{sign}\big(\hat{l}\big)$, and $\tau>0$ is a confidence margin such that only occupancy predictions with enough confidence are used. In our experiments, $\tau = 3$, corresponding to $\sigma(\tau) \approx 0.95$, and both loss terms enter the total training loss with unit weight.

Fig.~\ref{fig:method_overview} summarizes \oren. A compact set of key frames is maintained so that it covers the observed surface with little overlap between adjacent frames (Fig.~\hyperref[fig:method_overview]{2a}). From these key frames and the current frame, we generate a dataset of points that are: 1) on the surface, 2) perturbed, and 3) in the free-space (Fig.~\hyperref[fig:method_overview]{2b}). With the generated dataset, the model is optimized with a combination of reconstruction, Eikonal, projection, sign consistency, and occupancy losses (Fig.~\hyperref[fig:method_overview]{2e}). We refer the reader to \citet{dai_oren_2025} for the key-frame criterion, the sampling scheme, and the definitions and hyperparameters of the remaining losses.

The result is a continuously updated implicit map that returns accurate, non-truncated distance and gradient at any query point in the explored workspace. Its continuous clearance and Eikonal regularity \eqref{eq:sdf_constraints} are precisely the properties \bubblestar exploits next, turning signed-distance queries directly into safe, dynamically feasible flight corridors.

\section{Distance-Accelerated Motion Planning}
\label{sec:bubble_star}

In this section, we consider motion planning and trajectory optimization given an SDF representation of the environment. In order to solve \eqref{eq:planning_problem}, we split the problem into two parts: first finding a sequence of bubbles to the goal via global search to minimize path length, followed by local trajectory optimization within the bubble corridor to minimize the objective in \eqref{eq:planning_problem}. Since an SDF representation provides information not only about occupancy but also about the distance to the nearest occupied space, it can be used to accelerate motion planning by reducing the number of collision checks within large free-space regions. We refer to the ball of free space indicated by an SDF query as a \emph{bubble} and develop a new search-based motion planning algorithm, \bubblestar, to plan a sequence of bubbles from the start to the goal. \bubblestar constructs a graph of nodes on a grid, each associated with a bubble, and edges that connect to the bubble boundaries (Sec.~\ref{sec:bubble_star_graph}). Using a heuristic similar to A$^\star$, \bubblestar expands the most promising node at each step by adding the boundary nodes as successors (Sec.~\ref{sec:bubble_star_algorithm}). We prove completeness of \bubblestar under a mild clearance assumption (Sec.~\ref{sec:bubble_star_completeness}). The overlapping bubbles along the recovered path form a safe corridor that we use for trajectory optimization (Sec.~\ref{sec:bubble_star_trajectory}).

\subsection{Bubble Graph Construction}
\label{sec:bubble_star_graph}
For planning, we discretize the 3D environment of Sec.~\ref{sec:problem_statement} using grid resolution $\Delta > 0$ and define the set of grid nodes $\calV := \Delta \bbZ^3 \cap \calR$, where $\bbZ^3$ is the integer lattice in 3D. For each node $\bfp \in \calV$, we define the associated grid cell
$C_\Delta(\bfp) := \bfp + \left[-\frac{\Delta}{2}, \frac{\Delta}{2}\right]^3$,
and the set of all grid cells
$\calC_\Delta := \{ C_\Delta(\bfp) \mid \bfp \in \calV \}$.
We denote by $\calO := \{\, C_\Delta(\bfp) \in \calC_\Delta \mid C_\Delta(\bfp) \cap \Omega \neq \emptyset \,\}$ the occupied cells, i.e., those whose region intersects an obstacle, and by $\calF := \calC_\Delta \setminus \calO$ the free cells. The corresponding set of free grid nodes is
$\calV^\text{free} := \{ \bfp \in \calV \mid C_\Delta(\bfp) \in \calF \}$.
The free region can be represented explicitly using the SDF reconstructed using \oren in Sec.~\ref{sec:grad_sdf}. By computing the SDF $d(\bfc)$ at a node $\bfc \in \calV^\text{free}$, we define an open ball of radius at least $\frac{\Delta}{2}$. 
We refer to this obstacle-free region around the node as a bubble and denote it as $\bbB_\bfc = \bbB(\bfc, d(\bfc))$, where:
\begin{equation}\label{eq:ball}
	\bbB(\bfc,r) := \{ \bfp \in \bbR^3 \mid \norm{\bfp - \bfc}_2 < r\}.\!\!
\end{equation}
For the graph search over grid nodes, we use the free grid nodes that lie inside the bubble: 
\begin{equation}\label{eq:bubble}
	\calB_{\bfc} = \calB(\bfc,d(\bfc)) := \bbB_\bfc \cap \calV^\text{free}.
\end{equation}
\begin{figure}[t]
	\centering
	\includegraphics[width=0.95\linewidth,trim={50pt 5pt 50pt 5pt},clip]
	{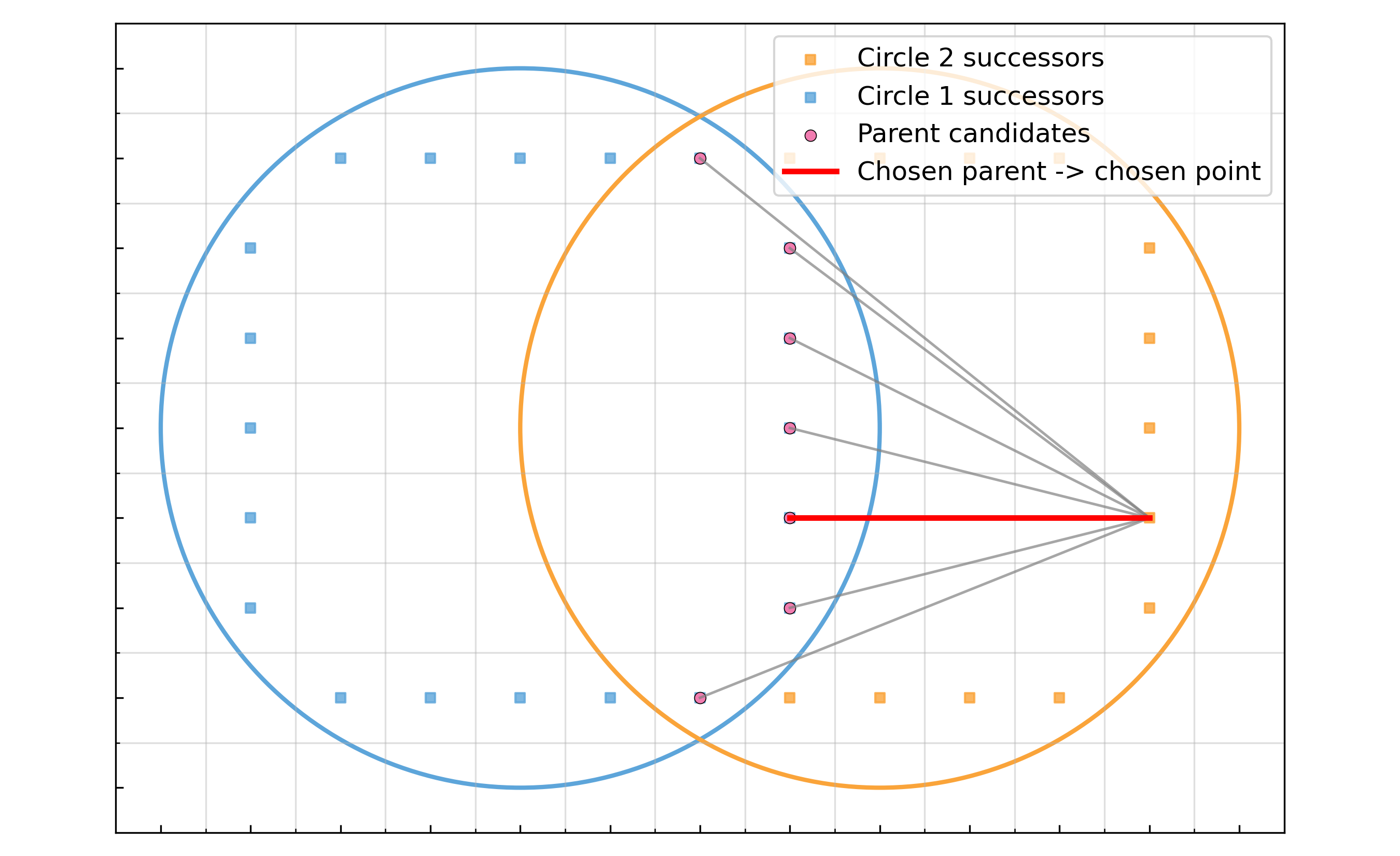}
	\caption{\small \textsc{CalculateSuccessors}: Each successor selects a parent in the current bubble according to the minimum cost-to-come and updates its cost. The OPEN set then contains both existing (blue) and new (orange) nodes.}
	\label{fig:successor-cost-prop}
\end{figure}
We define a graph over the free grid nodes $\calV^\text{free}$, in which each bubble connects the nodes it contains to the nodes on its boundary. The boundary nodes are characterized using the axis-aligned neighbors of a node:
\begin{equation}\label{eq:neighbors}
	\calN_{6}(\bfp) := \{ \bfp \pm \Delta \bfe_i \in \calV \mid i \in \{1, \dots, 3\} \},
\end{equation}
where $\{\bfe_i\}_{i=1}^{3}$ are the standard basis vectors of $\bbR^3$.
The \emph{successor set} of a bubble collects the boundary nodes, those with at least one axis-aligned neighbor \eqref{eq:neighbors} lying outside the bubble, as illustrated in Fig.~\ref{fig:successor-cost-prop}:
\begin{equation} \label{eq:successors}
	S(\calB_{\bfc}) = \left \{ \bfp \in \calB_{\bfc} \; \middle|\; \exists \bfq \in \calN_{6}(\bfp) \text{ s.t. } \bfq \notin \calB_{\bfc} \right\}.\!\!\!
\end{equation}
Two nodes $\bfp, \bfq \in \calV^\text{free}$ are connected by an edge if some bubble \eqref{eq:bubble} contains $\bfp$ and has $\bfq$ on its boundary. That is, $\bfp \in \calB_{\bfc}$ and $\bfq \in S(\calB_{\bfc})$ for some $\bfc \in \calV^\text{free}$. Every edge is collision-free: the bubble $\bbB_\bfc$ \eqref{eq:ball} is obstacle-free and convex, so it contains the straight line segment between the two nodes.

\subsection{\texorpdfstring{\bubblestar Search Algorithm}{Bubble* Search Algorithm}}
\label{sec:bubble_star_algorithm}

\begin{algorithm*}[t]
	\SetAlgoLined
	\LinesNumbered
	\DontPrintSemicolon
	\caption{\bubblestar Planner}\label{alg:bubble-star}
	\SetKwFunction{algo}{\bubblestar}
	\SetKwFunction{calcsucc}{CalculateSuccessors}
	\SetKwFunction{pathto}{PathTo}
	\SetKwInOut{Input}{Input}
	\SetKwProg{bstarprog}{Algorithm}{}{}
	\bstarprog{\algo{$\bfp_s$, $\bfp_g$, $d(\cdot)$, $h(\cdot)$}}{
		$\text{OPEN} \xleftarrow{} \{\bfp_s\}$, \quad $\text{CLOSED} \xleftarrow{} \emptyset$\;
		$g(\bfp_s) = 0$, \quad $g(\bfp) = \infty \ \ \forall \bfp \neq \bfp_s$\;
		\While{$\text{OPEN} \neq \emptyset$}{
			$\bfc \xleftarrow{} \argmin_{\bfp \in \text{OPEN}} \left( g(\bfp) + h(\bfp) \right)$\;
			$\text{OPEN} \xleftarrow{} \text{OPEN} \setminus \{\bfc\}$\;
			$\calB_{\bfc} \xleftarrow{} \calB(\bfc, d(\bfc))$\;
			\If{$\bfc = \bfp_g$}{
				\Return \pathto{$\bfp_g$}\;
			}
			$S_{\bfc} \xleftarrow{} S(\calB_{\bfc})$\;
			\If{$\bfp_g \in \calB_{\bfc}$}{
				$S_{\bfc} \xleftarrow{} S_{\bfc} \cup \{\bfp_g\}$\;
			}
			\calcsucc{$\calB_{\bfc}$, $S_{\bfc}$, $\text{OPEN}$, $\text{CLOSED}$}\;
			\ForEach{$\bfp \in \calB_{\bfc}$ \textnormal{with} $\bfp \notin S_{\bfc}$}{
				$\text{CLOSED} \xleftarrow{} \text{CLOSED} \cup \{\bfp\}$\;
			}
            $\text{CLOSED} \xleftarrow{} \text{CLOSED} \cup \{\bfc\}$\;
		}
		\Return failure\;
	}
    \BlankLine
	\SetKwProg{csprog}{Procedure}{}{}
	\csprog{\calcsucc{$\calB_{\bfc}$, $S_{\bfc}$, $\text{OPEN}$, $\text{CLOSED}$}}{
        \ForEach{$\bfj \in S_{\bfc}$ \textnormal{with} $\bfj \notin \text{CLOSED}$}{
		      $\bfk^\star \xleftarrow{} \argmin_{\bfk \in \text{OPEN} \cap \calB_{\bfc}} \left( g(\bfk) + \norm{\bfj - \bfk}_2 \right)$\;
			$g_{\min} \xleftarrow{} g(\bfk^\star) + \norm{\bfj - \bfk^\star}_2$\label{ln:best_parent}\; 
			\If{$g(\bfj) > g_{\min}$}{
				$g(\bfj) \xleftarrow{} g_{\min}$\;
				$\text{Parent}(\bfj) \xleftarrow{} \bfk^\star$, \quad $\text{Bubble}(\bfj) \xleftarrow{} \calB_{\bfc}$\;
				\uIf{$\bfj \in \text{OPEN}$}{
					update priority of $\bfj$\;
				}
				\Else{
					$\text{OPEN} \xleftarrow{} \text{OPEN} \cup \{\bfj\}$\;
				}
			}
		}
    }

	\BlankLine
	\SetKwProg{pathtoprog}{Procedure}{}{}
	\pathtoprog{\pathto{$\bfp$}}{
		$\text{Path} \xleftarrow{} \{ \bfp \}$\;
        $\text{BubblePath} \xleftarrow{} \{\,\calB(\bfp, d(\bfp))\,\}$\;

		$\text{current} \xleftarrow{} \bfp$\;
		\While{$\text{Parent}(\text{current})$ \textnormal{exists}}{
			$\text{Path} \xleftarrow{} \text{Path} \cup \{\text{current}\}$\;
            $\text{BubblePath} \xleftarrow{} \text{BubblePath} \cup \{\text{Bubble}(\text{current})\}$\;
			$\text{current} \xleftarrow{} \text{Parent}(\text{current})$\;
		}
		\Return $\{\text{Path.reverse()}, \text{BubblePath.reverse()} \}$\;
	}

\end{algorithm*}

We present the \bubblestar planner in Algorithm~\ref{alg:bubble-star}, which performs a heuristic graph search over bubbles to compute a collision-free path from the start node $\bfp_s$ to the goal node $\bfp_g$.
The algorithm expands bubbles constructed from SDF clearance, allowing large collision-free regions to be explored efficiently.  
This greatly reduces the required number of collision checks compared to other search-based planners.

\bubblestar maintains two sets of nodes: an OPEN set containing candidate successor nodes to be expanded, and a CLOSED set containing nodes whose associated bubbles have already been explored.
Each node $\bfp \in \calV^{\text{free}}$ is associated with a cost-to-come value $g(\bfp)$, representing the minimum path length from the start to $\bfp$, a parent pointer Parent$(\bfp)$ recording the predecessor node used to reach $\bfp$, and a bubble pointer, $\text{Bubble}(\bfp)$, recording the bubble that contains the segment from $\bfp$ to Parent$(\bfp)$. Analogous to an A$^\star$ search, the OPEN set is prioritized according to the sum of the cost-to-come and an optional heuristic estimate of the distance to the goal.

At each iteration, \bubblestar selects the node $\bfc \in \text{OPEN}$ with minimum priority and constructs the corresponding bubble $\calB_{\bfc} = \calB(\bfc, d(\bfc))$ defined in \eqref{eq:bubble}, along with its successor set $S_\bfc = S(\calB_{\bfc})$ in \eqref{eq:successors}. We refer to this as \emph{expanding} $\bfc$. If $\bfc = \bfp_g$, the search terminates. Otherwise, if the goal lies within $\calB_{\bfc}$, it is appended to the current successor set $S_{\bfc}$, so that the termination condition $\bfc = \bfp_g$ holds in a future iteration. The \textsc{CalculateSuccessors} procedure in Algorithm~\ref{alg:bubble-star} processes each successor $\bfj \in S_{\bfc}$: it selects the predecessor $\bfk^\star$ among nodes inside $\calB_{\bfc}$ that minimizes $g(\bfk) + \norm{\bfj - \bfk}_2$ and updates $g(\bfj)$, Parent$(\bfj)$, and Bubble$(\bfj)$ whenever this improves the cost-to-come. Interior nodes of $\calB_{\bfc}$ and its center $\bfc$ are then marked as CLOSED. This procedure assigns each successor node of the new bubble a parent node according to Line~\ref{ln:best_parent} of Algorithm~\ref{alg:bubble-star}.

The search terminates when the goal $\bfp_g$ is drawn from the open set (success), or if the open set is empty (failure). In the successful case, the parent pointers are followed to reconstruct a sequence of overlapping bubbles connecting the start and goal, which is done in the \textsc{PathTo} procedure. If the OPEN set becomes empty before reaching the goal, the algorithm reports failure.

Because the parent candidate $\bfk^\star$ may be any OPEN node inside the bubble, 
the recovered path can skip intermediate nodes unlike A$^\star$. \bubblestar therefore produces any-angle paths in that the nodes on the path are connected through one straight segment if they belong to the same bubble.
Therefore, the resulting paths are no longer than the shortest path restricted to the grid and are typically shorter.

\subsection{Termination, Completeness, and Failure Detection}
\label{sec:bubble_star_completeness}

We now establish formal guarantees for \bubblestar: it always terminates, it returns a path whenever one with sufficient clearance exists, and it reports failure otherwise. Throughout, we use the fact that the search space is bounded by $\calR$, so there are finitely many nodes $\calV$, and that the start and goal are free grid nodes, $\bfp_s, \bfp_g \in \calV^\text{free}$.
We assume the heuristic $h$ is consistent, i.e., $h(\bfp_g) = 0$ and $h(\bfp) \le \norm{\bfp - \bfp'}_2 + h(\bfp')$ for adjacent $\bfp, \bfp'$, which holds for the Euclidean distance $h(\bfp) = \norm{\bfp - \bfp_g}_2$ used in our experiments.

\begin{definition}[Clear grid path]\label{def:clear_grid_path}
    Let $\calV^\text{free}$ be a set of free nodes with resolution $\Delta$. A \emph{grid path} from $\bfp_s$ to $\bfp_g$ is a finite sequence $(\bfp_0, \bfp_1, \dots, \bfp_N) \subset \calV^\text{free}$ with $\bfp_0 = \bfp_s$, $\bfp_N = \bfp_g$, and $\bfp_{j+1} \in \calN_{6}(\bfp_j)$ for all $j$. It is \emph{clear} if the distance $d(\bfp_j) \ge 1.5\Delta$ for every $j$.
\end{definition}

The adjacency condition makes the sequence a connected walk on the grid, while the $1.5\Delta$ clearance ensures each node's bubble contains all of its grid neighbors. Because the axis-aligned neighbors lie at distance $\Delta$ and bubbles are open balls, the radius must satisfy $d(\bfp_j) > \Delta$. Among the clearance values admissible under our grid-occupancy model, $1.5\Delta$ is the smallest. The completeness proof below relies on this property.

\begin{theorem}[Termination]
	\label{thm:termination}
	\bubblestar (Algorithm~\ref{alg:bubble-star}) terminates after finitely many iterations.
\end{theorem}

\begin{proof}
	The grid $\calV$ is finite and each bubble is determined by its center, so there are finitely many bubbles. With a consistent heuristic, the cost-to-come $g(\bfc)$ is optimal when $\bfc$ is expanded \citep{heuristic_basis}, so each node is expanded at most once and is never reinserted into OPEN after entering CLOSED. Each iteration removes one node from OPEN and inserts finitely many successors, so after at most $|\calV|$ expansions OPEN is empty or the goal is returned.
\end{proof}

\begin{theorem}[Completeness]
	\label{thm:completeness}
	If a clear grid path from $\bfp_s$ to $\bfp_g$ exists, then \bubblestar (Algorithm~\ref{alg:bubble-star}) returns a valid path.
\end{theorem}

\begin{proof}
	By Theorem~\ref{thm:termination}, the search terminates, so it either returns a path or reports failure. Suppose, for contradiction, that it reports failure. Then, OPEN is empty at termination. Let $(\bfp_0, \dots, \bfp_N)$ be the clear grid path from the hypothesis, with $\bfp_0 = \bfp_s$, $\bfp_N = \bfp_g$, $\bfp_{j+1} \in \calN_{6}(\bfp_j)$, and $d(\bfp_j) \ge 1.5\Delta$ for all $j$. We prove by induction that every $\bfp_j$ lies in some bubble expanded during the search.

	For the base case, $\bfp_0 = \bfp_s$ is placed in OPEN at initialization. Since OPEN is empty at termination, it was popped and its bubble expanded, so $\bfp_0$ lies in an expanded bubble. For the inductive step, suppose $\bfp_j$ lies in an expanded bubble $\calB_\bfc$.
	If $\bfp_j$ is interior to $\calB_\bfc$, that is $\bfp_j \notin S(\calB_\bfc)$, then by \eqref{eq:successors} all grid neighbors of $\bfp_j$, including $\bfp_{j+1}$, lie in $\calB_\bfc$.

	Otherwise $\bfp_j$ lies on the boundary, $\bfp_j \in S(\calB_\bfc)$, and was inserted into OPEN. Since OPEN is empty, it was also popped and its own bubble $\calB_{\bfp_j}$ is expanded. Because $d(\bfp_j) \ge 1.5\Delta > \Delta$ and $\bfp_{j+1}$ is a grid neighbor at distance $\Delta$, we have $\bfp_{j+1} \in \calB_{\bfp_j}$. In either case $\bfp_{j+1}$ lies in an expanded bubble, completing the induction.

	In particular, the goal $\bfp_N = \bfp_g$ lies in some expanded bubble $\calB_\bfc$. The goal test $\bfp_g \in \calB_\bfc$ in Algorithm~\ref{alg:bubble-star} then inserts $\bfp_g$ into the successor set, so \textsc{CalculateSuccessors} inserts it into OPEN with finite cost. Since OPEN is empty at termination, $\bfp_g$ must have been popped. When a node equal to $\bfp_g$ is popped, the algorithm returns \textsc{PathTo}$(\bfp_g)$. This contradicts the failure assumption, hence, \bubblestar returns a valid path.
\end{proof}

\begin{corollary}[Failure detection]
	\label{cor:failure}
	\bubblestar (Algorithm~\ref{alg:bubble-star}) reports failure in finite time and only when no clear grid path connects $\bfp_s$ and $\bfp_g$.
\end{corollary}

\begin{proof}
	Since Theorem~\ref{thm:termination} shows termination in finite time and Theorem~\ref{thm:completeness} shows completeness, it follows that if no path exists, the algorithm terminates with failure in finite time.
\end{proof}

\subsection{Trajectory Optimization Within Bubble Corridor}
\label{sec:bubble_star_trajectory}

Given the sequence of bubbles returned by \bubblestar, we compute a trajectory to solve \eqref{eq:planning_problem} that lies within the convex bubble corridor.
Pairs of consecutive bubbles define a convex overlap region through which the trajectory may pass. These overlap regions provide a sequence of safe sets that constrain the trajectory while preserving sufficient freedom for optimization.

\paragraph*{Bubble Overlap Construction}
\bubblestar returns a sequence of bubbles \eqref{eq:ball} along the path, which we index as $\bbB_{i} := \bbB(\bfc_i, r_i)$ with centers $\bfc_i \in \bbR^3$ and radii $r_i := d(\bfc_i) > 0$. Consecutive bubbles $\bbB_i$ and $\bbB_{i+1}$ overlap by construction, since \bubblestar generates each bubble center inside its predecessor. Their intersection is a lens whose widest cross-section is a disk in the plane perpendicular to the line between centers $\bfa_i := \bfc_{i+1} - \bfc_i$. We use this disk as a convex safe set linking the two bubbles. Its offset along $\bfa_i$, center, and radius are:
\begin{equation}\label{eq:overlap_geometry}
	\begin{aligned}
		\lambda_i & = \frac{1}{2} + \frac{r_i^2 - r_{i+1}^2}{2 \norm{\bfa_i}^2}, \\
		\bfo_i    & = \bfc_i + \lambda_i \bfa_i,                                 \\
		\rho_i    & = \sqrt{r_i^2 - \lambda_i^2 \norm{\bfa_i}^2}.
	\end{aligned}
\end{equation}
Let $\bfnu_i = \bfa_i / \norm{\bfa_i}$ be the unit normal of the disk's plane and $B_i = \begin{bmatrix} \bfe_1^i & \bfe_2^i \end{bmatrix} \in \bbR^{3 \times 2}$ an orthonormal basis of that plane. With the center $\bfo_i$ and radius $\rho_i$ from \eqref{eq:overlap_geometry}, the widest overlapping cross-section is the disk
\begin{equation}\label{eq:overlap_region}
	\begin{aligned}
		\calS_i & = \{\, \bfo_i + B_i \bfw \mid \bfw \in \bbR^2,\ \norm{\bfw}_2 \le \rho_i \,\} \\
		        & \subset \bbB_i \cap \bbB_{i+1}.
	\end{aligned}
\end{equation}

\paragraph*{MINCO Trajectory Representation}
We build upon the MINCO representation~\citep{MINCO}, which parameterizes a piecewise polynomial trajectory using intermediate waypoints and segment durations. Let the trajectory consist of $N$ polynomial segments. MINCO parameterizes the trajectory using spatial variables $\bfxi = \{\bfxi_i\}_{i=1}^{N-1}$ and temporal variables $\bftau = \{\tau_i\}_{i=1}^{N}$, which determine the intermediate waypoints and segment durations, respectively. These variables are mapped to physical waypoints $\bfq = \bfq(\bfxi)$ and segment durations $\bfT = \bfT(\bftau)$, and the optimal polynomial coefficients $\bfc$ are uniquely determined from $(\bfq, \bfT)$.

The trajectory optimization problem is formulated as
\begin{equation}
	\min_{\bfxi, \bftau}
	\calJ(\bfq(\bfxi), \bfT(\bftau)) \quad \text{s.t.} \quad \bfq_i(\bfxi_i) \in \calS_i,\;\forall i,
\end{equation}
where $\calJ$ includes the minimum-snap and time penalties of \eqref{eq:planning_problem} as well as corridor violation penalties. The dynamics constraints are penalized through the cost, following \citet{MINCO}. The clearance constraint in \eqref{eq:planning_problem} is enforced by shrinking the radii of all bubbles in the corridor by the required clearance $r$, and the rest of the constraints can be written naturally in the MINCO problem formulation, so the optimization is \eqref{eq:planning_problem} restricted to the MINCO polynomial class.

We design a smooth unconstrained parameterization that automatically satisfies the constraint of each waypoint lying within its corresponding overlap region of \eqref{eq:overlap_region}, i.e., $\bfq_i \in \calS_i$.
Specifically, we map unconstrained variables $\bfxi_i \in \bbR^2$ to constrained waypoints $\bfq_i$ as:
\begin{equation}
	\begin{aligned}
		\bfu_i & = \frac{B_i \bfxi_i}{\sqrt{1 + \norm{B_i \bfxi_i}^2}}, \\
		\bfq_i & = \bfo_i + \rho_i \bfu_i.
	\end{aligned}
\end{equation}
Segment durations must satisfy the positivity constraint $T_i > 0$.
Following the temporal constraint elimination scheme of MINCO~\citep{MINCO}, we introduce unconstrained temporal variables $\tau_i \in \bbR$ and define the mapping:
\begin{equation}
	T_i =
	\begin{cases}
		\frac{\tau_i^2}{2} + \tau_i + 1,                    & \tau_i > 0,   \\
		\left(\frac{1}{2}\tau_i^2 - \tau_i + 1\right)^{-1}, & \tau_i \le 0,
	\end{cases}
\end{equation}
which ensures positivity while preserving differentiability. Gradients with respect to $\tau_i$ are computed using the chain rule in conjunction with the gradients provided by the MINCO formulation.

\paragraph*{Gradient-Based Optimization}

MINCO provides gradients of the cost function with respect to waypoints and durations, $\partial \calJ / \partial \bfq$ and $\partial \calJ / \partial \bfT$. These gradients are propagated through the spatial parameterization using
\begin{equation}
	\frac{\partial \calJ}{\partial \bfxi_i}
	=
	B_i^\top \rho_i
	\left(
	\alpha I_{3 \times 3} - \alpha \bfu_i \bfu_i^\top
	\right)
	\frac{\partial \calJ}{\partial \bfq_i},
\end{equation}
where $\alpha = (1 + \norm{B_i \bfxi_i}^2)^{-1/2}$.
To improve convergence, we first prune the bubbles whose waypoints $\bfq_i$ make no progress toward the goal, and initialize the remaining waypoints at the centers of their overlap regions $\bfo_i$. The optimization then runs in two stages:
\begin{enumerate}
	\item Optimize the spatial variables $\bfxi$ with the durations $\bftau$ fixed to obtain a collision-free geometric path.
	\item Jointly optimize $\bfxi$ and $\bftau$ to obtain a dynamically feasible trajectory.
\end{enumerate}
This two-stage optimization produces smooth dynamically feasible trajectories through the sequence of bubbles from the \bubblestar algorithm.

\section{Evaluation}
\label{sec:evaluation}

Our central claim is that combining \oren and \bubblestar yields an efficient integrated mapping, planning, and trajectory optimization approach for UAVs that runs fully onboard.
To support this claim, we first show that \oren provides an accurate continuous SDF representation with sufficient computational efficiency to permit real-time operation in simulated large-scale environments. 
Our prior work \citep{dai_oren_2025} compares \oren extensively against SDF mapping baselines, and Sec.~\ref{sec:sdf-comparison} summarizes those results. The key remaining comparison is against the mapping approaches commonly used in autonomous flight, and we evaluate \oren in four large simulated environments against OctoMap \citep{hornung13auro}, a widely used mapping library for quadrotors (Sec.~\ref{sec:eval_sim_mapping}). 
Then, building on the SDF reconstruction from \oren, another set of experiments shows that \bubblestar is more efficient than grid-based A$^\star$ and sampling-based (RRT, RRT$^\star$) planners (Sec.~\ref{sec:eval_planning}).
Finally, we evaluate our approach on a real quadrotor with a Jetson Orin NX computer to show that it runs efficiently and fully onboard (Sec.~\ref{sec:eval_onboard}).

\subsection{Comparison with SDF Mapping Methods}
\label{sec:sdf-comparison}

In our conference paper \citep{dai_oren_2025}, \oren is evaluated on Replica \citep{replica19arxiv} and Newer College \citep{newercollege2021} datasets, in comparison with four baselines \htwomapping \citep{jiang_h2-mapping_2023}, PIN-SLAM \citep{pan_pin-slam_2024}, HIO-SDF \citep{hio-sdf_2024} and Voxblox \citep{oleynikova_voxblox_2017}. \oren generates mesh results of quality similar to the baselines' for small indoor Replica scenes and synthesized depth data. However, on the large outdoor Newer College scene with real LiDAR measurements, \oren reconstructs a much better mesh, as shown in Table~\ref{tab:mesh_metrics_newercollege}. On the Newer College dataset, \oren improves over the best baselines by 29--68\% across all seven mesh metrics (Completion $+36.2\%$, Completion Ratio $+29.3\%$, Recall $+33.5\%$, Precision $+40.3\%$, F1 $+37.0\%$, Chamfer-L1 $+56.7\%$, Accuracy $+68.3\%$), averaging $\approx 43.1\%$.
As for SDF metrics, \oren outperforms the baselines by a larger margin on both Replica and Newer College datasets. As shown in Table~\ref{tab:sdf_metrics_improvement}, \oren provides at least $15\%$ more accurate SDF than the baselines, which is an essential factor for the success of our \bubblestar planner.

\begin{table*}[t]
    \centering
    \caption{\small Mesh reconstruction metrics with $\delta=20$~cm on the Newer College dataset \citep{newercollege2021}.}
    \label{tab:mesh_metrics_newercollege}
    \begin{tabular}{l|ccccc}
    \hline
    Metric & \oren & \htwomapping & PIN-SLAM & HIO-SDF & Voxblox \\
    \hline
    Completion [cm] $\downarrow$        & \textbf{10.66} & 21.94 & \underline{16.71} & 72.86  & 21.30 \\
    Completion Ratio [$<\delta$]\% $\uparrow$    & \textbf{94.20} & 61.58 & \underline{72.83} & 10.05  & 60.31 \\
    Recall [$<\delta$]\% $\uparrow$              & \textbf{93.99} & 57.96 & \underline{70.40} &  4.72  & 56.64 \\
    Precision [$<\delta$]\% $\uparrow$           & \textbf{90.69} & 52.97 & \underline{64.63} &  4.46  & 51.84 \\
    F1 Score [$<\delta$]\% $\uparrow$            & \textbf{92.31} & 55.35 & \underline{67.39} &  4.59  & 54.14 \\
    Chamfer-L1 Distance [cm] $\downarrow$ & \textbf{9.36} & 28.40 & \underline{21.64} & 422.29 & 23.37 \\
    Accuracy [cm] $\downarrow$          & \textbf{8.07}  & 34.86 & 26.58 & 771.72 & \underline{25.44} \\
    \hline
    \end{tabular}
\end{table*}

\begin{table}[t]
  \centering
  \caption{\small Average improvement of \oren over other SDF estimation methods (the best baseline for each scene and metric). SDF MAE is the MAE of the predicted signed distance, and Gradient MAE is the MAE of the angle between the predicted and ground-truth SDF gradients.}
  \label{tab:sdf_metrics_improvement}
  \setlength{\tabcolsep}{4pt}
  \begin{tabular*}{\columnwidth}{@{\extracolsep{\fill}}llr}
    \toprule
    Metric & Region & Avg.\ $\uparrow$ \\
    \midrule
    \multirow{3}{*}{SDF MAE}  & All  & 22.0\% \\
                              & Near & 34.5\% \\
                              & Far  & 15.0\% \\
    \midrule
    \multirow{3}{*}{Gradient MAE} & All  & 35.5\% \\
                                & Near & 54.1\% \\
                                & Far  & 29.9\% \\
    \bottomrule
  \end{tabular*}
\end{table}

\subsection{Mapping in Large Simulated Environments}
\label{sec:eval_sim_mapping}
The comparisons in Sec.~\ref{sec:sdf-comparison} establish \oren's accuracy against dedicated SDF mapping methods on standard benchmarks. This section evaluates the performance of \oren in an online setting with data streaming from a UAV in large simulated environments. 

We build four simulation environments, a forest, an underground garage, an industrial site, and a warehouse, and simulate a quadrotor flying through each, carrying a single onboard depth camera and building the map online from the streaming point clouds.
These scenes form our testbed for the mapping-and-planning approach. Here, we evaluate the map representation 
and Sec.~\ref{sec:eval_planning} evaluates planning in the same environments.
With our focus being on online deployment of autonomous robots, we compare against OctoMap~\citep{hornung13auro}, an occupancy-grid mapping method widely used on compute-constrained platforms. \oren matches OctoMap's occupancy prediction quality at a comparable update rate while providing the continuous distances that accelerate planning while recovering visually better surface meshes. 

\begin{table}[tbp]
	\centering
	\caption{\small Comparison of \oren and OctoMap on $4$ simulated environments. Both methods are compared against a ground-truth surface point cloud.}
	\label{tab:sim_map_comparison}
	\footnotesize
\begin{tabular*}{\columnwidth}{@{\extracolsep{\fill}}llrrr}
	\hline
	Env & Method & Precision & Recall & F1 \\
	\hline
	\multirow{2}{*}{Warehouse}
	& \oren & $0.672$ & $\mathbf{0.884}$ & $\mathbf{0.764}$ \\
	& OctoMap & $\mathbf{0.780}$ & $0.629$ & $0.696$ \\
	\hline
	\multirow{2}{*}{Garage}
	& \oren & $0.862$ & $\mathbf{0.911}$ & $\mathbf{0.886}$ \\
	& OctoMap & $\mathbf{0.888}$ & $0.819$ & $0.852$ \\
	\hline
	\multirow{2}{*}{Forest}
	& \oren & $0.652$ & $\mathbf{0.948}$ & $\mathbf{0.773}$ \\
	& OctoMap & $\mathbf{0.900}$ & $0.384$ & $0.539$ \\
	\hline
	\multirow{2}{*}{Industrial}
	& \oren & $0.556$ & $\mathbf{0.901}$ & $0.688$ \\
	& OctoMap & $\mathbf{0.787}$ & $0.646$ & $\mathbf{0.710}$ \\
	\hline
\end{tabular*}
\end{table}

Table~\ref{tab:sim_map_comparison} reports occupancy prediction quality with the \emph{Precision}, \emph{Recall}, and \emph{F1} metrics common in the mapping literature \citep{jiang_h2-mapping_2023, pan_pin-slam_2024, dai_oren_2025}, evaluated against the ground-truth surface point cloud. \oren achieves better or comparable results across the four environments while additionally providing distance information and gradients. Fig.~\ref{fig:mesh_compare_sim} visualizes the reconstructed surface meshes side by side in each scene, the marching-cubes mesh of the zero level-set of \oren's SDF and the faces of OctoMap's occupied cells.

OctoMap builds a discrete occupancy grid of the free and occupied voxels at a fixed resolution. \oren instead represents the scene as a continuous SDF that provides the signed distance and gradient at every point (Sec.~\ref{sec:sdf_residual}). Throughout the evaluation, both the occupancy predictions and the reconstructed meshes are derived from the SDF estimate $\hat{d}$ alone. The occupancy decoder could provide occupancy directly, but we choose to measure the representational power of the SDF itself, though the occupancy information is used during online training to improve the SDF training. Compared to OctoMap, \oren produces a more accurate and complete reconstruction of the scene, with fewer holes and artifacts.

\begin{figure*}[tp]
	\centering
	\begin{subfigure}[t]{0.49\linewidth}
		\includegraphics[width=\linewidth]{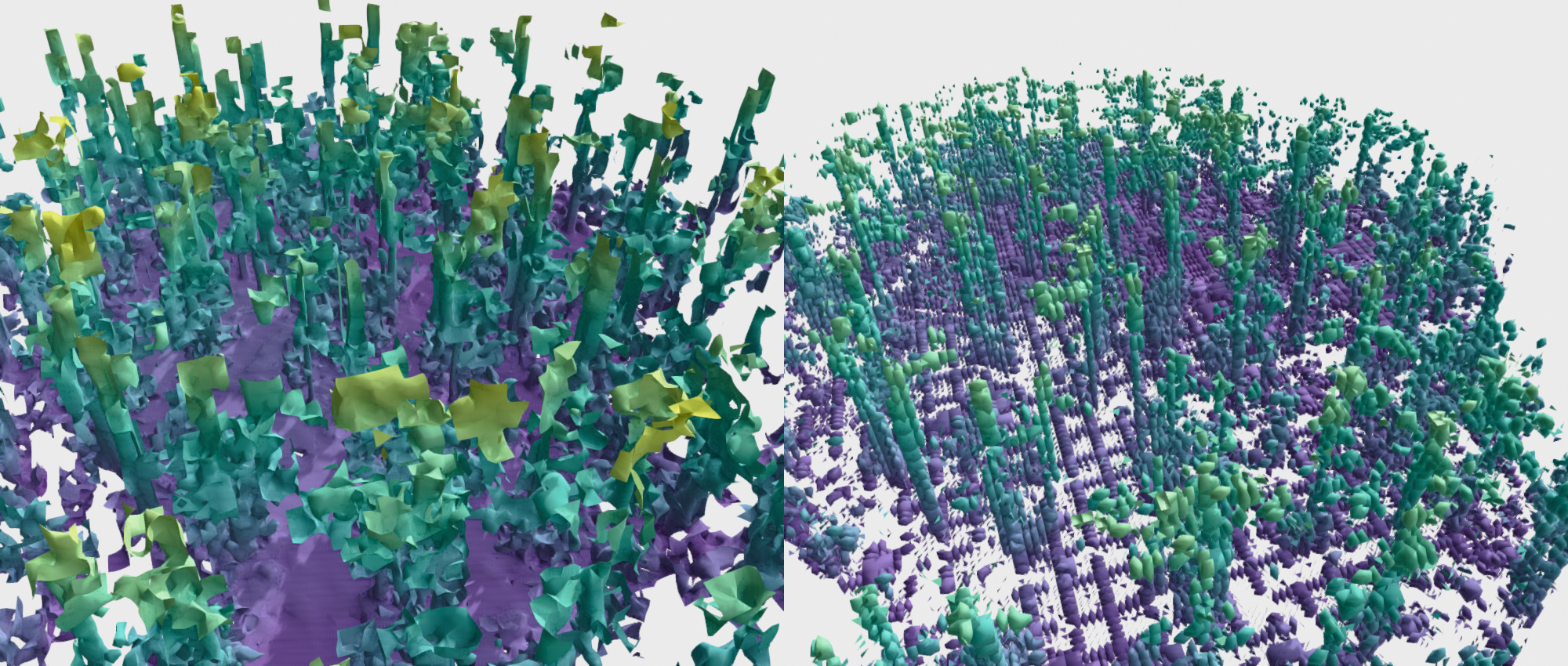}
		\caption{\footnotesize Forest}
		\label{fig:mesh_compare_forest}
	\end{subfigure}
	\hfill
	\begin{subfigure}[t]{0.49\linewidth}
		\includegraphics[width=\linewidth]{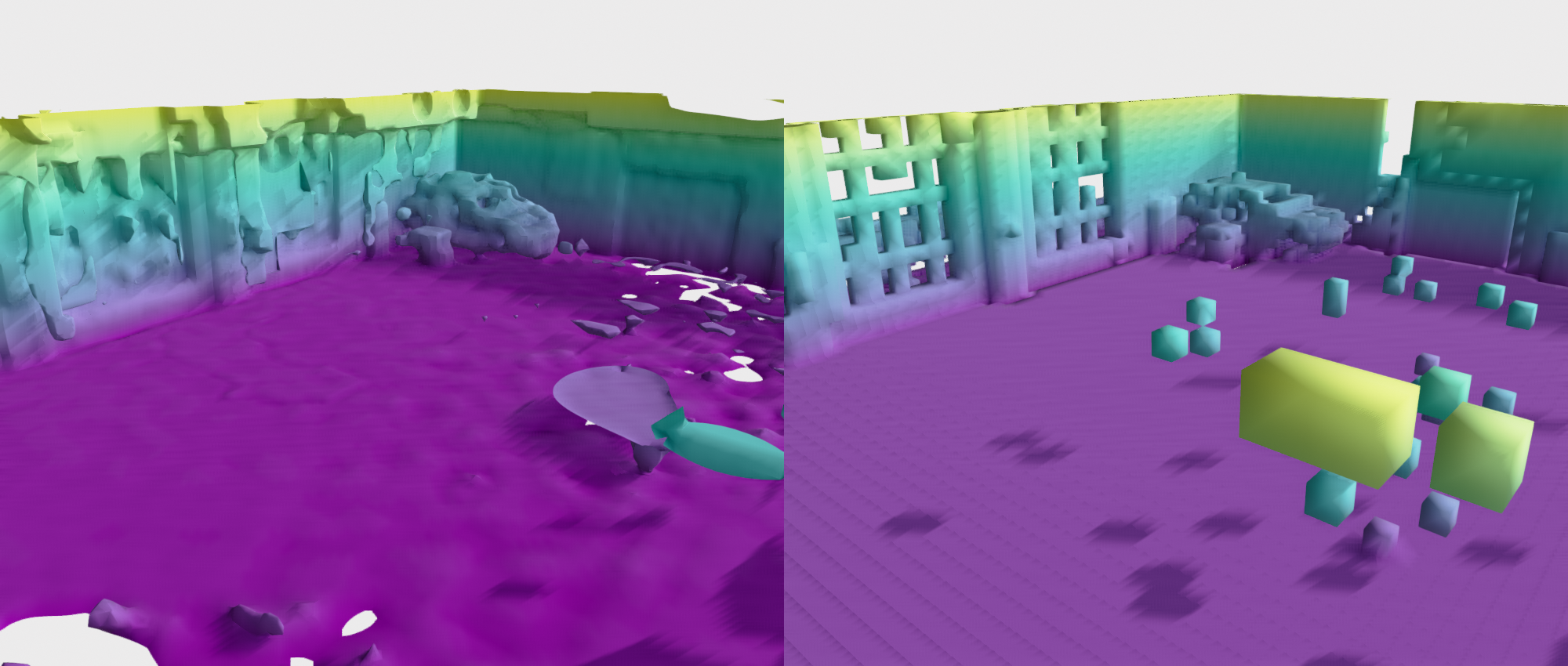}
		\caption{\footnotesize Garage}
		\label{fig:mesh_compare_garage}
	\end{subfigure}

	\begin{subfigure}[t]{0.49\linewidth}
		\includegraphics[width=\linewidth]{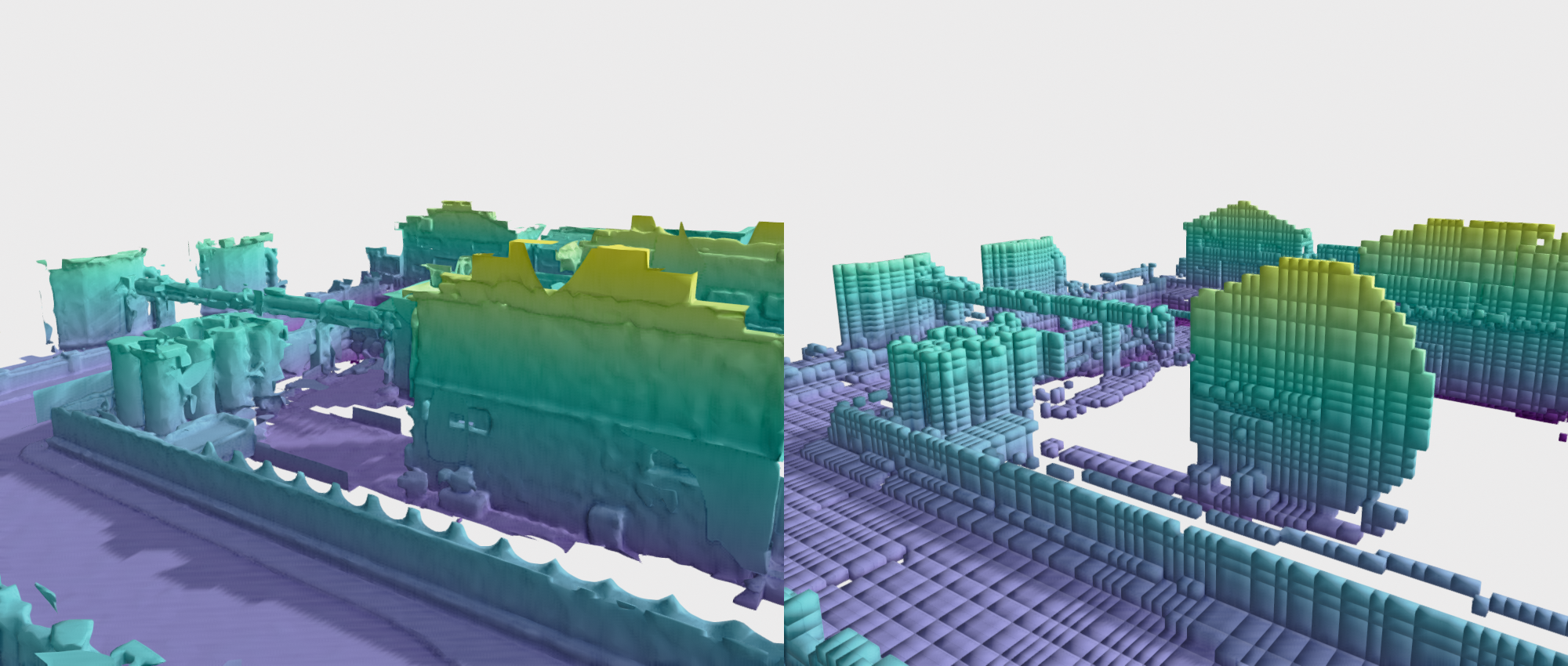}
		\caption{\footnotesize Industrial}
		\label{fig:mesh_compare_industrial}
	\end{subfigure}
	\hfill
	\begin{subfigure}[t]{0.49\linewidth}
		\includegraphics[width=\linewidth]{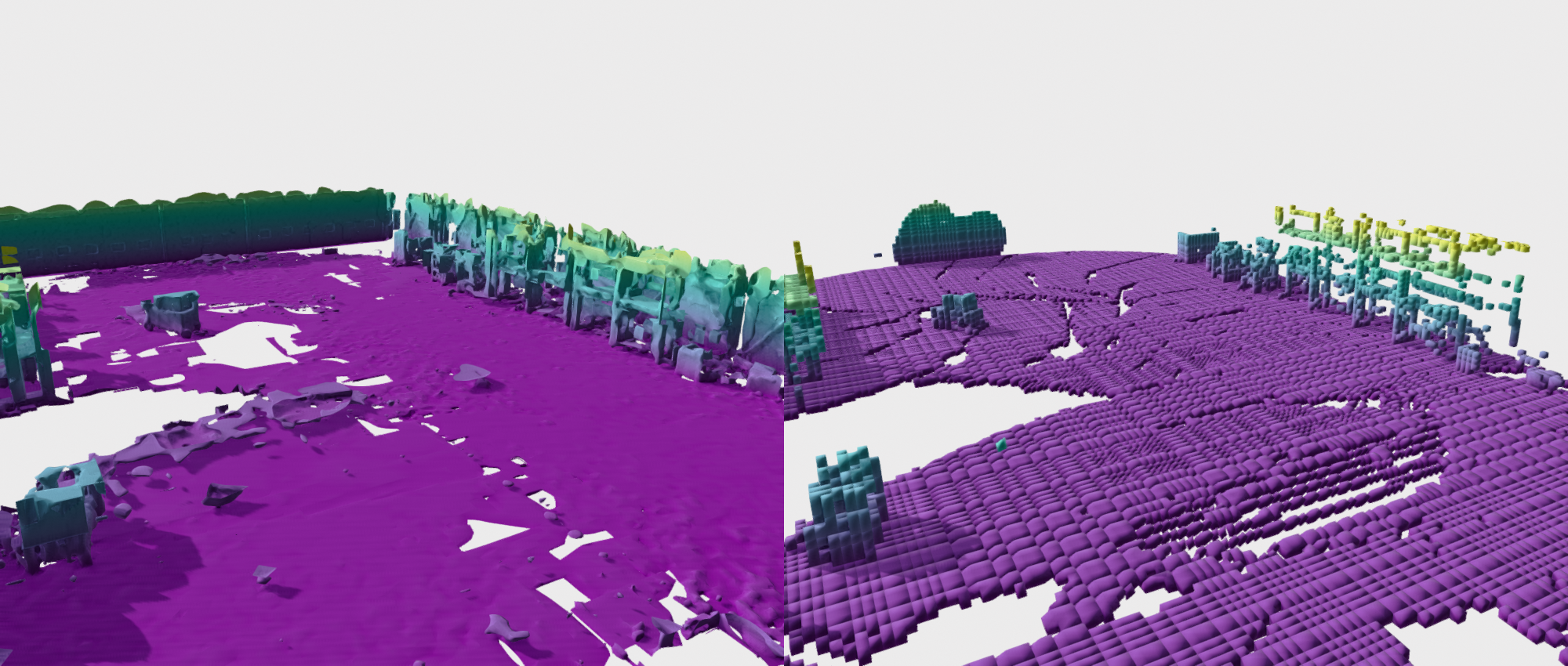}
		\caption{\footnotesize Warehouse}
		\label{fig:mesh_compare_warehouse}
	\end{subfigure}
	\caption{\small Map representations compared across four large-scale simulated environments. In each pair, \oren (left) reconstructs a continuous SDF surface and OctoMap~\citep{hornung13auro} (right) a discretized occupancy grid of the same scene. Each map is built online from a single depth camera on a simulated quadrotor flown through a sequence of fixed waypoints.}
	\label{fig:mesh_compare_sim}
\end{figure*}

\begin{figure}[t]
	\centering
	\includegraphics[width=\linewidth]{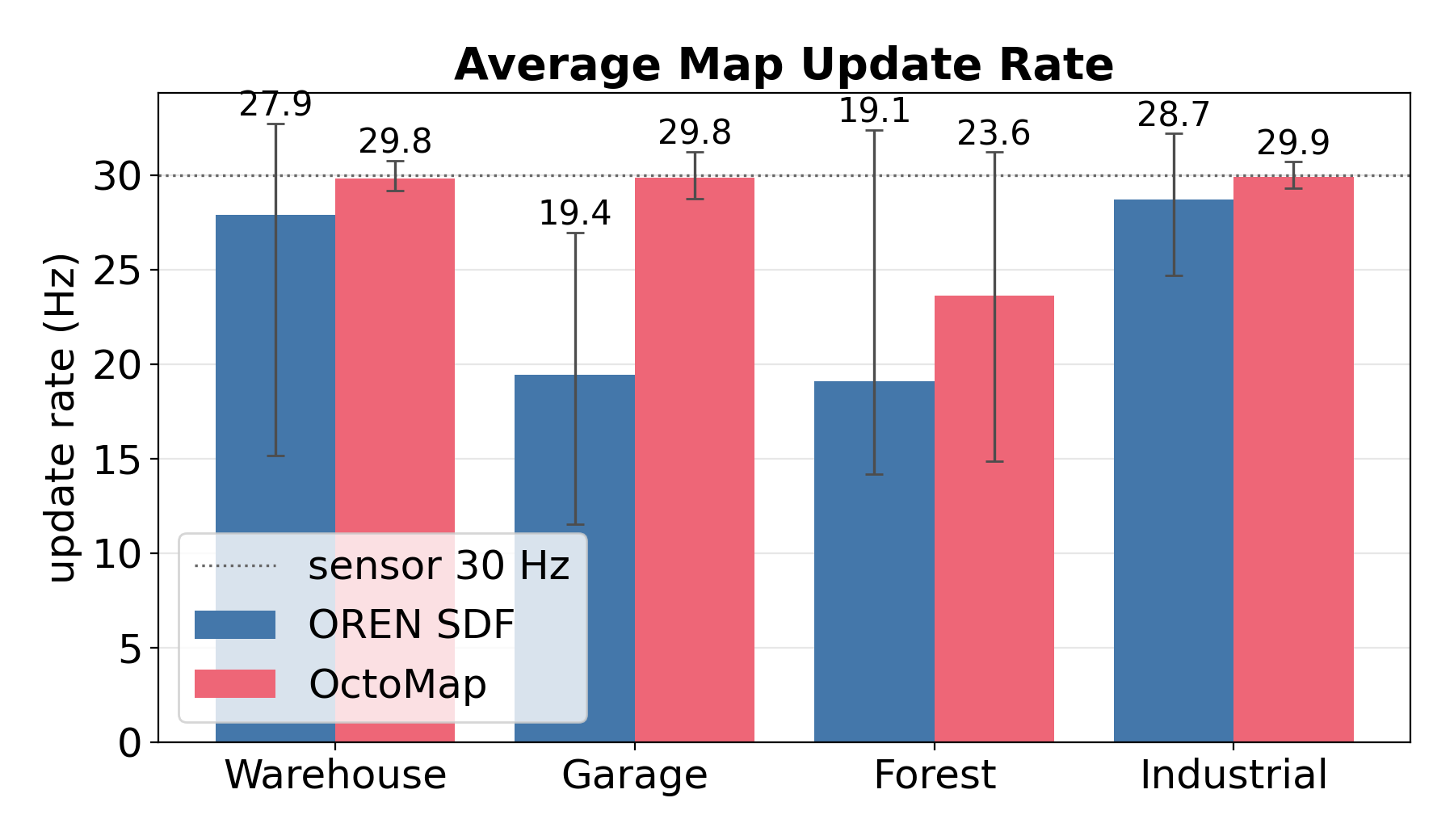}
	\caption{\small Average mapping update rate (Hz) of \oren and OctoMap~\citep{hornung13auro} across the four simulated environments. The depth camera streams at $30$~Hz (dotted line). Both run in real time. OctoMap is faster, as updating an occupancy grid is lighter than fitting an implicit field, whereas \oren provides the smooth, differentiable SDF that answers the distance queries \bubblestar depends on.}
	\label{fig:map_speed}
\end{figure}

The continuous SDF also carries richer information for planning. It encodes the distance to the nearest obstacle that \bubblestar uses for its efficiency and safety. This added information comes at only a modest expense in computation time. As Fig.~\ref{fig:map_speed} shows, \oren yields an update frequency comparable to OctoMap's, only around $1$~Hz slower in two of the four environments (warehouse and industrial), and $4$--$10$~Hz slower in the others (garage and forest).
The map update rate ($19.1$--$28.4$~Hz) remains amenable to real-time operation.
Compared to the benefit of the continuous distance information, we regard the increase in computation as a small price to pay, since it enables \bubblestar, which OctoMap cannot.

\begin{figure*}[tp]
	\centering
	\begin{subfigure}[t]{0.48\linewidth}
		\includegraphics[width=\linewidth]{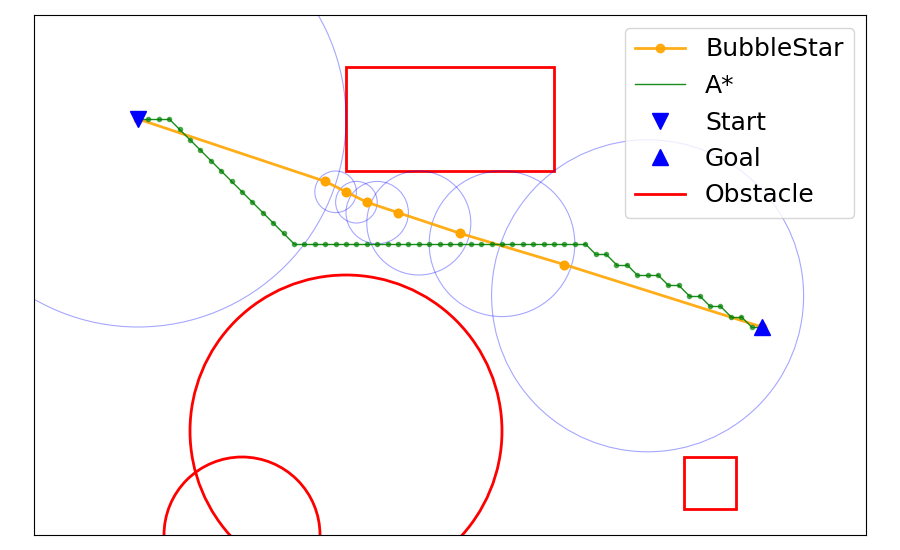}
		\caption{}
		\label{fig:astar_env}
	\end{subfigure}%
	\hfill%
	\begin{subfigure}[t]{0.5\linewidth}
		\includegraphics[width=\linewidth]{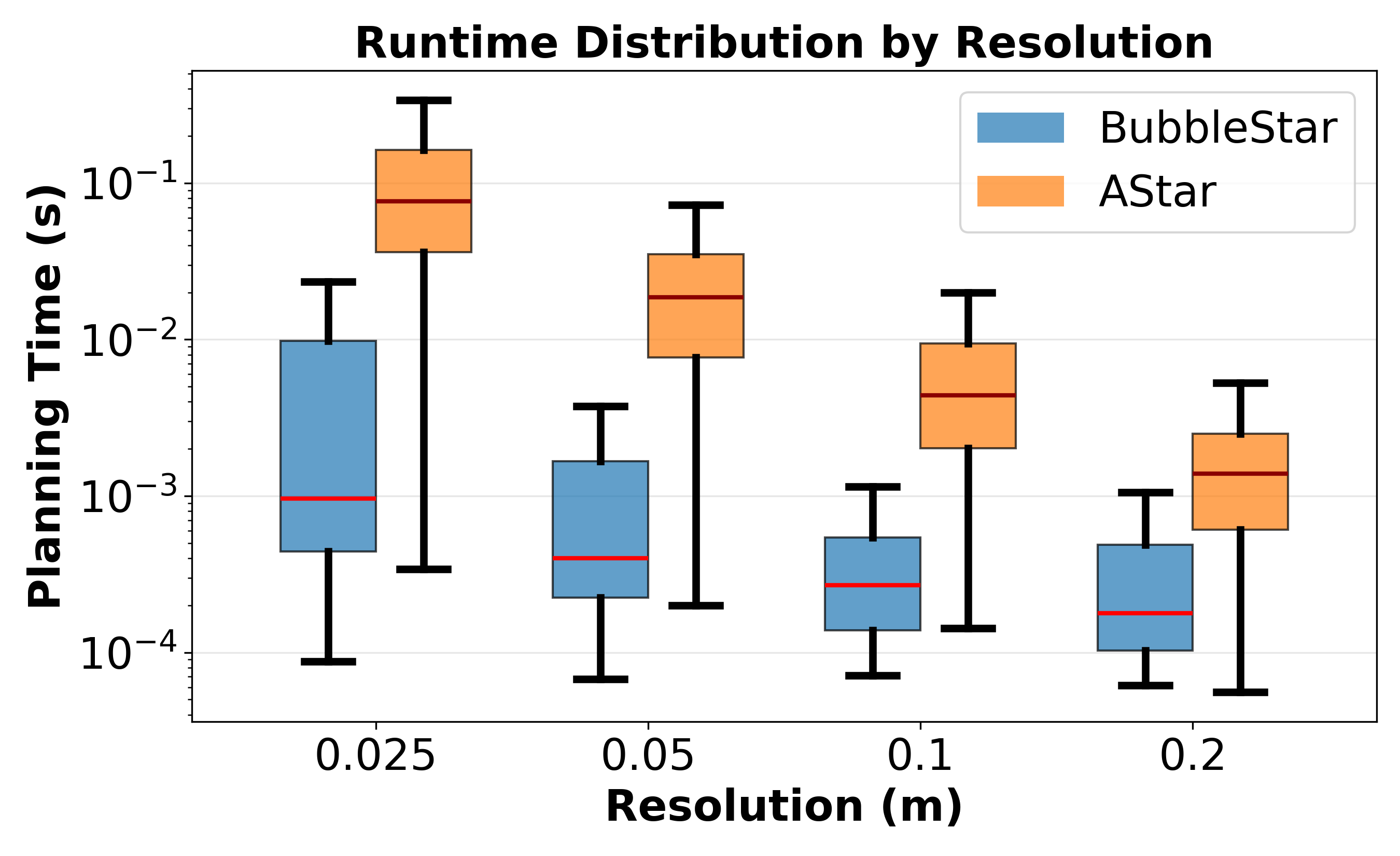}
		\caption{}
		\label{fig:astar_runtime}
	\end{subfigure}
	\caption{\small 2D search-efficiency comparison between \bubblestar and A$^\star$. (a) Representative paths in a 2D environment; SDF clearance lets \bubblestar take any-angle shortcuts that grid-restricted A$^\star$ cannot. (b) Runtime over 1000 randomly sampled start--goal pairs at several grid resolutions. Boxes span the interquartile range, the red line marks the median, and whiskers extend to the minimum and maximum.}
	\label{fig:astar_comparison}
\end{figure*}

\subsection{Planning Evaluation}
\label{sec:eval_planning}

We evaluate the efficiency of \bubblestar by the number of collision checks the planner issues and by the total planning time, and report path length and tracking error to verify that efficiency does not come at the cost of trajectory quality. We first isolate search efficiency in a controlled 2D comparison against grid-based A$^\star$, and then evaluate the complete planner in the 3D industrial and forest environments of Sec.~\ref{sec:eval_sim_mapping} against grid-based (A$^\star$) and sampling-based (RRT, RRT$^\star$) baselines. 

\paragraph*{Search Efficiency in 2D}

We first evaluate \bubblestar against A$^\star$ in a 2D environment shown in Fig.~\ref{fig:astar_comparison}. 
A$^\star$ binarizes the same SDF to determine occupancy and expands one neighboring node at a time, whereas \bubblestar uses the SDF magnitude to grow collision-free bubbles and expands only at their boundaries. Both planners use the Euclidean distance to the goal, $h(\bfp) = \norm{\bfp - \bfp_g}_2$, as the heuristic. Over $1000$ trials at several resolutions, \bubblestar performs $91$--$99\%$ fewer collision checks than A$^\star$ and finds a shorter or equal path in every case. Fig.~\ref{fig:astar_runtime} shows the runtime distribution, where \bubblestar is about an order of magnitude faster, with a wider gap at finer resolutions.

\begin{figure*}[t]
	\centering
	\begin{subfigure}[t]{0.49\linewidth}
		\includegraphics[width=\linewidth]{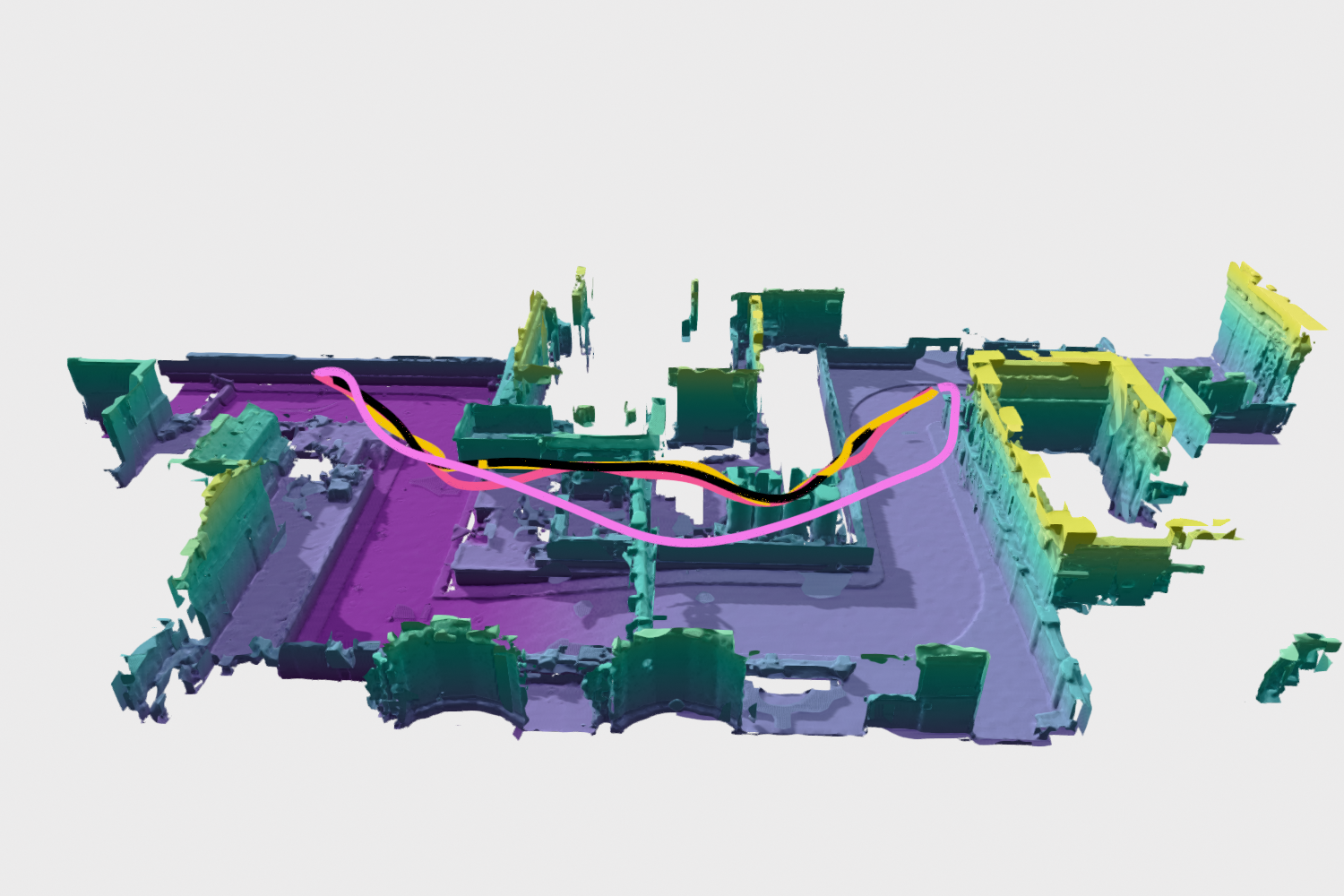}
		\caption{\footnotesize Planner trajectories}
		\label{fig:industrial_planner_comparison}
	\end{subfigure}
	\hfill
	\begin{subfigure}[t]{0.49\linewidth}
		\includegraphics[width=\linewidth]{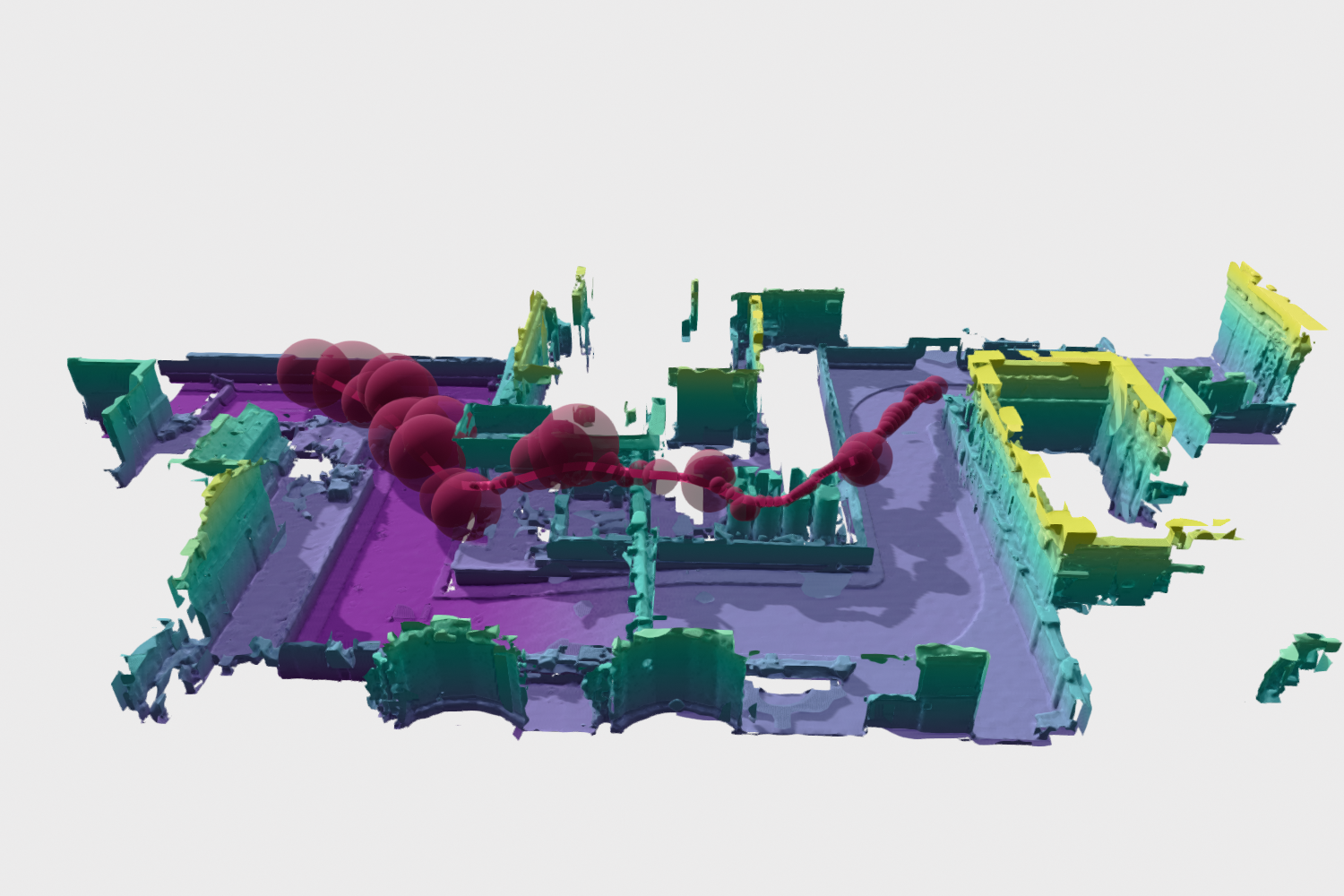}
		\caption{\footnotesize \bubblestar safe corridor}
		\label{fig:industrial_bubbles}
	\end{subfigure}
	\caption{\small Planning in the simulated industrial environment of Fig.~\ref{fig:mesh_compare_industrial}, all on the same \oren SDF. (a) Representative trajectories from \bubblestar ({\color{red}red}), A$^\star$ ({\color{orange}orange}), RRT ({\color{magenta}magenta}), and RRT$^\star$ (black). (b) The \bubblestar safe corridor: overlapping bubbles (red) grown along the optimized trajectory, whose large, mutually overlapping clearance gives the optimizer ample room and keeps the optimization fast (Table~\ref{tab:planner_comparison}).}
	\label{fig:industrial_planning}
\end{figure*}
\begin{table*}[tb]
	\centering
	\caption{\small Planning comparison in the simulated industrial and forest environments of Fig.~\ref{fig:mesh_compare_sim}. All planners operate on the same SDF map. Columns report the planning time (Search), the trajectory-optimization time (Opt.), their sum (Total), the path length (Path), the root-mean-square position error of the executed trajectory relative to the planned trajectory (RMSE), the number of overlapping bubbles forming the safe corridor (Bubbles), and their mean radius (Mean $r$). \textbf{Bold} marks the best value in each column and \underline{underline} the second best, per environment. For the Bubbles and Mean $r$ columns, which characterize the corridor handed to the trajectory optimizer, fewer bubbles and larger mean radius are better. \bubblestar attains the lowest total planning time in both environments while producing a compact high-clearance corridor.}
	\label{tab:planner_comparison}
	\footnotesize
	\setlength{\tabcolsep}{6pt}
	\begin{tabular*}{\textwidth}{@{\extracolsep{\fill}}llrrrrrrr}
		\hline
		Env & Planner & Search (ms) & Opt.\ (ms) & Total (ms) & Path (m) & RMSE (m) & Bubbles & Mean $r$ (m) \\
		\hline
		\multirow{4}{*}{Industrial}
		& \bubblestar & $485$ & $\mathbf{523}$ & $\mathbf{1008}$ & \underline{$86.0$} & $\mathbf{0.291}$ & $\mathbf{82}$ & \underline{$1.08$} \\
		& A$^\star$ & \underline{$223$} & $1005$ & \underline{$1228$} & $90.3$ & \underline{$0.297$} & $164$ & $0.55$ \\
		& RRT & $\mathbf{2}$ & $6068$ & $6069$ & $102.5$ & $0.327$ & $\mathbf{82}$ & $\mathbf{1.17}$ \\
		& RRT$^\star$ & $5079$ & \underline{$927$} & $6006$ & $\mathbf{85.4}$ & $0.314$ & \underline{$123$} & $0.69$ \\
		\hline
		\multirow{4}{*}{Forest}
		& \bubblestar & $\mathbf{298}$ & \underline{$3478$} & $\mathbf{3776}$ & $\mathbf{87.2}$ & \underline{$0.274$} & $\mathbf{311}$ & $\mathbf{0.30}$ \\
		& A$^\star$ & \underline{$466$} & $3687$ & \underline{$4152$} & \underline{$89.5$} & $\mathbf{0.265}$ & $341$ & $0.27$ \\
		& RRT & $2931$ & $\mathbf{3056}$ & $5987$ & $90.1$ & $0.281$ & $348$ & $0.27$ \\
		& RRT$^\star$ & $5050$ & $5318$ & $10368$ & $94.7$ & $0.324$ & \underline{$328$} & \underline{$0.28$} \\
		\hline
	\end{tabular*}
\end{table*}

\paragraph*{Planning in 3D Environments}

We compare \bubblestar in 3D environments against three baselines: grid-based A$^\star$ and the sampling-based RRT and RRT$^\star$ planners. To generate dynamically feasible trajectories for a quadrotor from the baselines, we place bubbles along the geometric paths found by the baseline planners, and use the same optimization routine as in Sec.~\ref{sec:bubble_star_trajectory}. All planners use the same 3D SDF obtained using \oren. We evaluate on the industrial site and the forest scene environments from Sec.~\ref{sec:eval_sim_mapping}.
A representative trajectory from each planner in the industrial scene is shown in Fig.~\ref{fig:industrial_planner_comparison}, and quantitative metrics for both scenes are reported in Table~\ref{tab:planner_comparison}.

In the industrial environment, \bubblestar achieves the lowest total planning time ($1.0$~s) while matching the shortest path within $1\%$ ($86.0$~m vs $85.4$~m by RRT$^\star$), well below A$^\star$ ($90.3$~m) and RRT ($102.5$~m).
The marginally shorter path by RRT$^\star$ is expected because RRT$^\star$ routes arbitrarily close to obstacles in continuous space, whereas \bubblestar searches on a grid with fixed resolution and keeps its trajectory inside high-clearance bubbles, trading a fraction of length for larger clearance margin.
 
The two planning stages, search and optimization, trade off differently across the methods. A$^\star$ searches quickly ($223$~ms) but the high number of bubbles created around the path ($164$) leads to a dense cell corridor and slow trajectory optimization ($1005$~ms). RRT finds a discrete path almost instantly ($2$~ms). However, because the path is jagged, it takes a long time ($6068$~ms) to optimize a continuous trajectory within bubbles centered at the RRT path. RRT$^\star$ requires a stopping criterion for search, and we chose one that is based on finding a path shorter than \bubblestar, but this causes RRT$^\star$ to spend most of its budget on search ($5.1$~s).
\bubblestar keeps both stages low ($485$~ms search, $523$~ms optimization). Tracking error is comparable across all planners, with \bubblestar achieving the lowest tracking error ($0.291$~m), confirming that its corridor does not compromise dynamic feasibility.

By construction, \bubblestar produces a corridor of overlapping safe regions, which keeps the subsequent trajectory optimization fast. The Bubbles and Mean $r$ columns of Table~\ref{tab:planner_comparison} quantify this: \bubblestar covers its path with $82$ overlapping bubbles of mean radius $1.08$~m (Fig.~\ref{fig:industrial_bubbles}), whereas A$^\star$'s dense cell path requires $164$ bubbles of roughly half the radius ($0.55$~m).
Fewer and larger bubbles leave the optimizer more room around each waypoint and fewer convex regions to stitch together, which is why \bubblestar optimizes in $523$~ms against the $1005$~ms A$^\star$ requires.
Larger bubbles alone are not sufficient, however: RRT attains an even larger mean radius ($1.17$~m) with the same bubble count, yet its randomly sampled path threads them along an irregular route and is the slowest of all to optimize ($6.1$~s).
\bubblestar is the only method that pairs a compact, high-clearance corridor with a smooth underlying path, which keeps both its search and its optimization fast.

This pattern still holds in the more cluttered forest environment. \bubblestar again attains the lowest total planning time ($3.8$~s) and produces the shortest path ($87.2$~m), with the most compact, highest-clearance corridor among the four planners ($311$ bubbles at a $0.30$~m mean radius).
The sampling-based planners remain the most expensive: RRT$^\star$ spends over $10$~s across search and optimization, and both produce longer paths than \bubblestar. 

Across both experiments, the efficiency of \bubblestar follows directly from the tight coupling between map and planner: by reading clearance straight from \oren's non-truncated SDF map and expanding whole collision-free bubbles rather than testing occupancy cell by cell, it issues far fewer signed-distance/collision queries than A$^\star$ while preserving the safety and completeness guarantees (Sec.~\ref{sec:bubble_star_completeness}).

\subsection{Real-World Autonomous Flight}
\label{sec:eval_onboard}

\begin{figure}[t]
	\centering
	\begin{subfigure}{\linewidth}
		\centering
		\includegraphics[width=\linewidth, trim=0cm 10cm 20cm 8cm, clip]{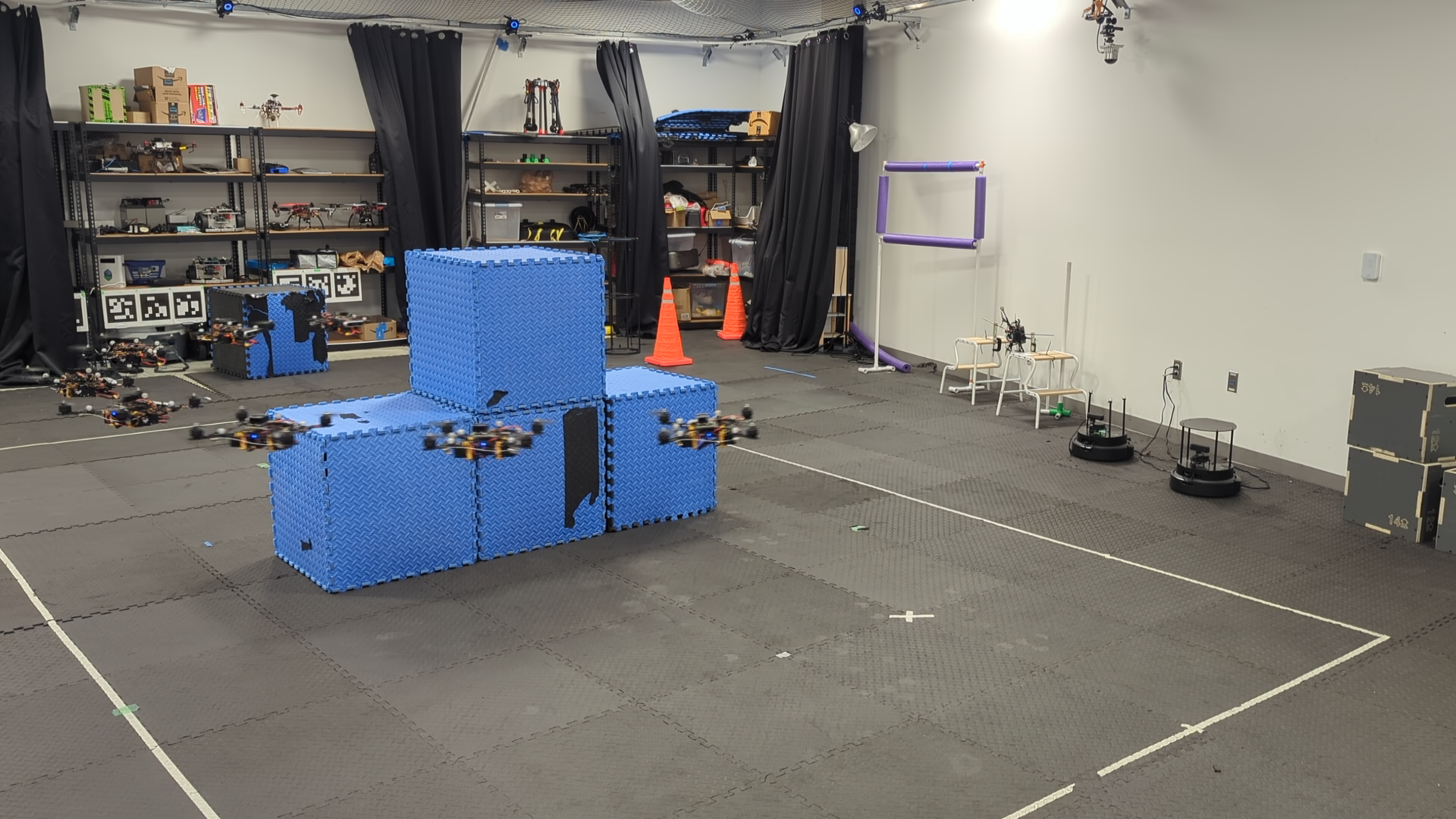}
		\caption{}
		\label{fig:mapping_run}
	\end{subfigure}
	\begin{subfigure}{\linewidth}
		\centering
		\includegraphics[width=\linewidth, trim={100pt 200pt 100pt 300pt}, clip]{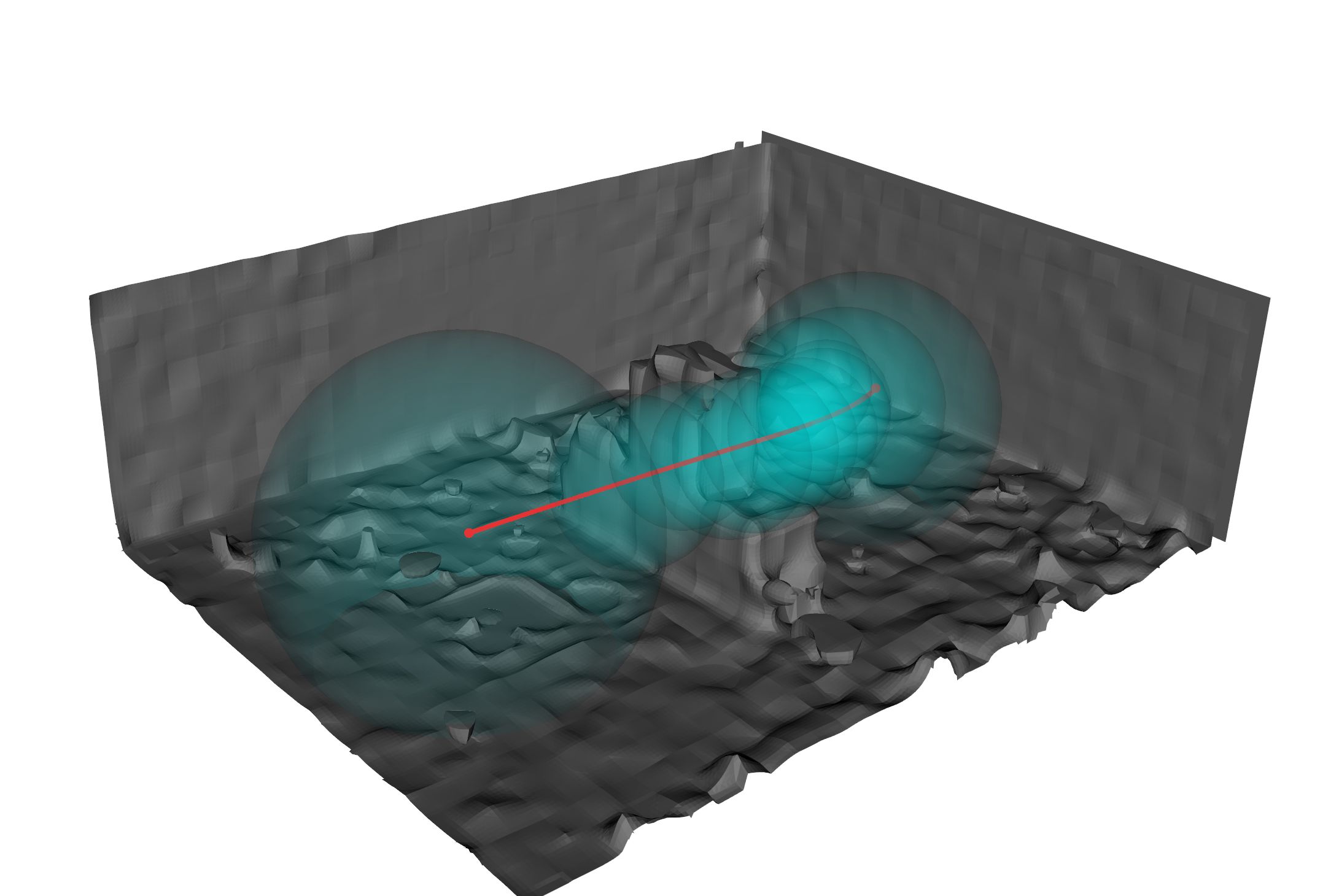}
		\caption{}
		\label{fig:onboard_flight}
	\end{subfigure}
	\caption{\small Fully onboard flight in an indoor environment: (a) A quadrotor circles a set of obstacles with yaw fixed along the acceleration direction; (b) \oren reconstructs SDF online (shown as gray mesh extracted from the surface), \bubblestar plans a bubble corridor (cyan) and optimizes a quadrotor trajectory (red), all computed in real time.}
	\label{fig:onboard_results}
\end{figure}

\begin{figure}[t]
	\centering
    \begin{subfigure}{\linewidth}
		\centering
        \includegraphics[width=\linewidth,trim=0cm 0cm 0cm 5cm, clip]{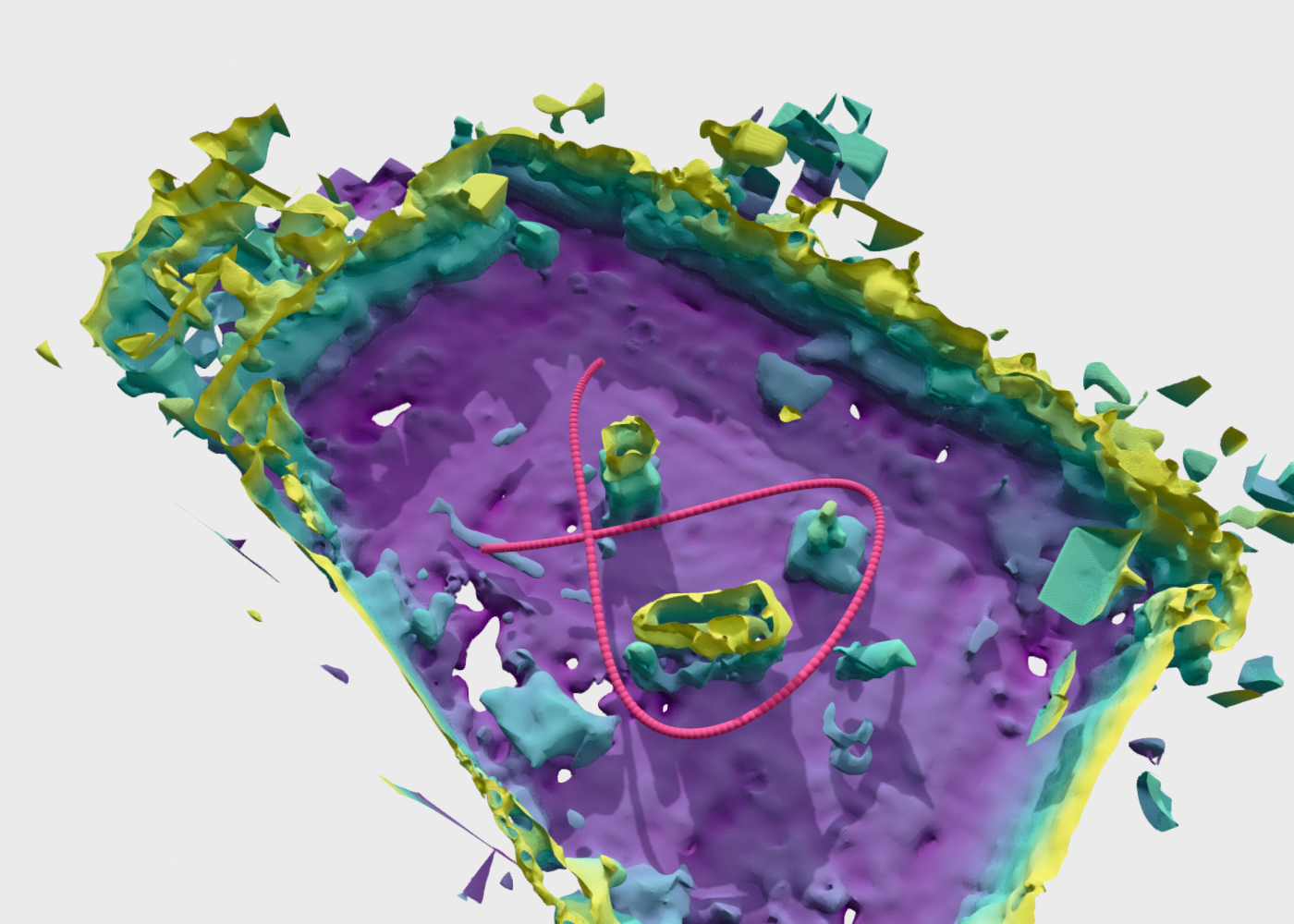}
		\caption{}
		\label{fig:onboard_route_image}
	\end{subfigure}
	\begin{subfigure}{\linewidth}
		\centering
		\includegraphics[width=\linewidth]{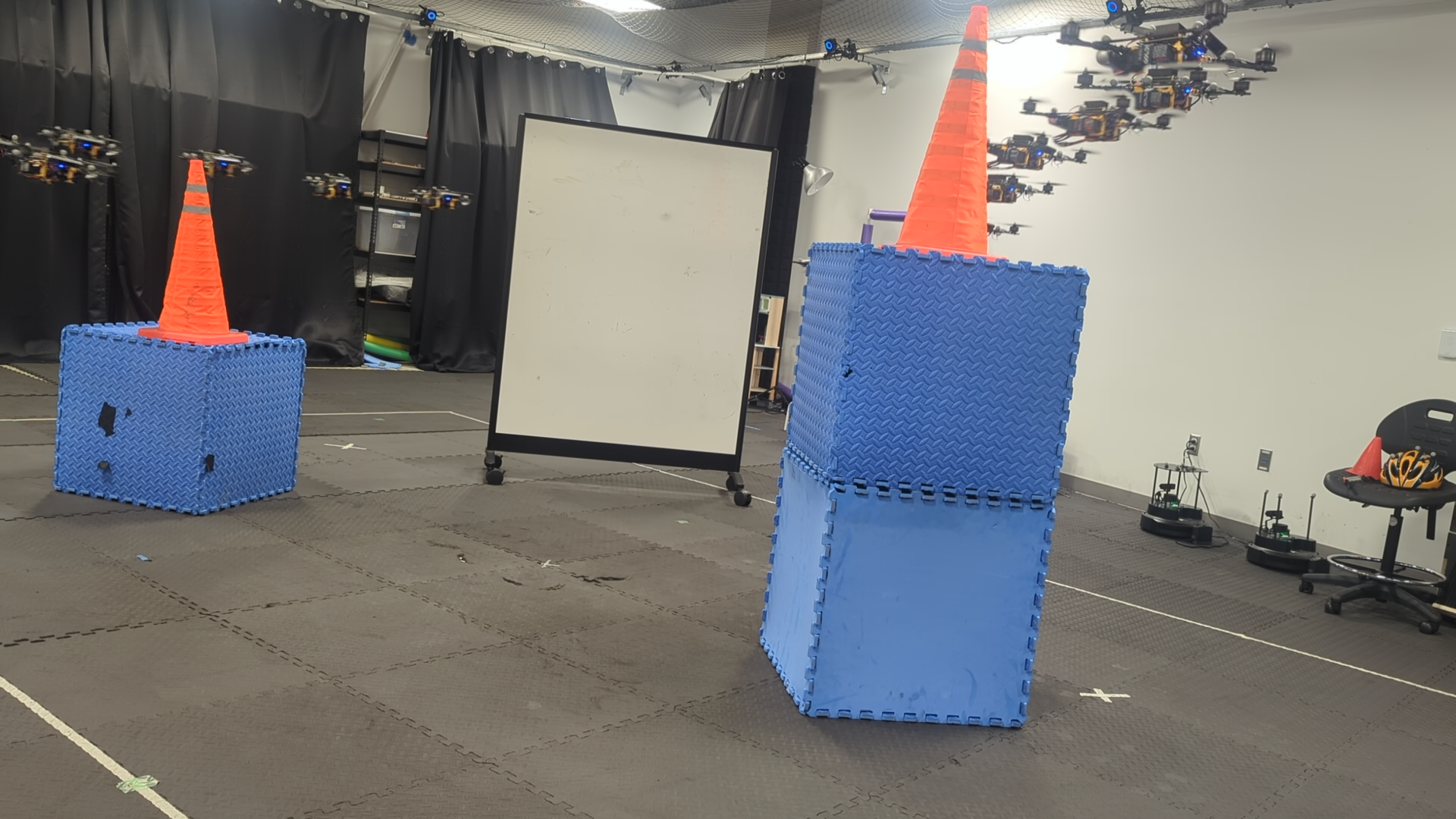}
		\caption{}
		\label{fig:onboard_traj_flight}   
	\end{subfigure}
	\caption{\small (a) Mesh extracted from \oren SDF onboard the quadrotor in real time and a looping trajectory (red) obtained by concatenating several \bubblestar plans into a single bubble sequence. (b) The robot executing the planned trajectory.}
\end{figure}

We have shown that \oren produces a non-truncated SDF efficiently (Sec.~\ref{sec:eval_sim_mapping}) and \bubblestar plans safe dynamically feasible trajectories using the non-truncated SDF values (Sec.~\ref{sec:eval_planning}). We now validate our integrated approach in real-world flight, with \oren mapping and \bubblestar planning running onboard a quadrotor equipped with a Jetson Orin NX. The UAV navigates a previously unseen indoor environment, building the SDF online from streaming depth measurements and planning safe corridors in real time.

Fig.~\ref{fig:mapping_run} shows a representative deployment of our method. \oren runs at $7$~Hz to provide the mesh shown in Fig.~\ref{fig:onboard_flight}. \bubblestar uses the SDF values from \oren to generate the trajectory shown in Fig.~\ref{fig:onboard_flight}. Beyond point-to-point planning, several \bubblestar plans can be concatenated into a single sequence of overlapping bubbles, letting the quadrotor follow a longer, more complex route while preserving the per-segment clearance guarantees of Sec.~\ref{sec:bubble_star_completeness}. Fig.~\ref{fig:onboard_route_image} shows a multi-segment trajectory executed onboard, threading a looping path through the SDF mesh reconstructed during the flight. This trajectory is created by selecting three waypoints and sequentially planning from one to the next. We then concatenate these sequences of bubbles and find a single trajectory through the entire route. Table~\ref{tab:real_planning_results} shows statistics on computation time, clearances, and trajectory lengths for this multi-segment trajectory. The planner maintains a safe distance from obstacles as it navigates around them. Fig.~\ref{fig:onboard_traj_flight} visualizes the robot flying this trajectory.

These results show that combining an SDF map with a planner that exploits SDF to construct safe corridors for trajectory optimization makes safe autonomous flight tractable under tight compute and memory budgets.

\begin{table}[t]
  \centering
  \caption{\small Runtime and tracking performance for the trajectory shown in Fig.~\ref{fig:onboard_traj_flight}. Clearance values are shown after accounting for the robot footprint.}
  \label{tab:real_planning_results}
  \setlength{\tabcolsep}{4pt}
  \begin{tabular*}{\columnwidth}{@{\extracolsep{\fill}}lr}
    \toprule
    Planning time (\bubblestar)      & $76.96$~ms \\
    Optimization time                & $142.33$~ms \\
    Total plan time                  & $219.30$~ms \\
    \midrule
    Min.\ obstacle distance, planned & $0.313$~m \\
    Min.\ obstacle distance, flown   & $0.169$~m \\
    Flown clearance (p5 / median)    & $0.297$ / $0.804$~m \\
    \midrule
    Segments / pieces                & $3$ / $39$ \\
    Trajectory duration              & $17.38$~s \\
    \bottomrule
  \end{tabular*}
\end{table}

\section{Conclusion}
\label{sec:conclusion}

This paper developed an efficient unified approach for SDF reconstruction and SDF-accelerated motion planning for safe autonomous flight.
Our mapping method, \oren, reconstructs differentiable non-truncated SDF online by combining an explicit octree prior with an implicit neural residual, attaining the accuracy of neural network methods at the memory and runtime efficiency of volumetric ones.
Exploiting the SDF representation, \bubblestar constructs a graph of collision-free bubbles, which forms a safe corridor with guarantees of termination, completeness, and failure detection.
Such a bubble corridor defines safety constraints for dynamically feasible trajectory optimization. Compared with baselines, \bubblestar is the only method that pairs a compact high-clearance corridor with a smooth underlying path, which keeps both its search and its trajectory optimization fast.

We demonstrate the integrated mapping and planning approach running fully onboard a quadrotor in real time. This shows that an accurate implicit map representation makes autonomous flight tractable under tight onboard resources. In our experiments, \oren maintains real-time speed while supplying the continuous differentiable SDF that occupancy-grid mappers cannot, and \bubblestar exploits this representation to issue $91$--$99\%$ fewer collision checks than A$^\star$ and to attain the lowest total planning time while matching the shortest path. Together, these results enable the complete mapping and planning loop to run efficiently onboard a Jetson Orin NX.

\section*{Funding}
We gratefully acknowledge support from ARL DCIST CRA W911NF-17-2-0181 (N. Atanasov, Z. Dai), a research gift fund established by Shield AI (J. Stanley), and the Ministry of Trade, Industry and Energy (MOTIE), Korea, under the Strategic Technology Development Program, supervised by Korea Institute for Advancement of Technology (KIAT) [Grant No. P0026052] (K.M.B. Lee).

\section*{Statements and Declarations}
\textbf{Conflict of interest} \ The authors have no competing interests to declare that are relevant to the content of this article.

\textbf{Data Availability} \ A public demonstration of \oren and \bubblestar, including code to run the integrated mapping and planning approach, is available at \url{https://github.com/ExistentialRobotics/erl_oren_bubble_star_demo}.

\textbf{Author Contributions} OREN extension concepts: ZD, NA; Developing OREN: ZD, QQ, TF; Planner concept: JS, KMBL; Developing planning algorithm: JS, TH; Quadrotor experiments: JS, TH, QQ, SS, CB; Writing main manuscript: JS, ZD; Writing, review and editing: KMBL, NA; Supervision: NA. 
\bibliography{bib/ref}

\end{document}